\documentclass[twoside]{article}

\usepackage[accepted]{aistats2022}
\usepackage[round]{natbib}

\usepackage[hidelinks]{hyperref}
\usepackage{url}
\usepackage{graphicx,subfig}
\usepackage{tikz}
\usepackage{todonotes}
\usepackage{amsfonts,amsmath,dsfont}
\usepackage{cuted}
\usepackage{amsthm}

\newtheorem{theorem}{Theorem}

\newcommand{\D}{\mathrm{d}}
\newcommand{\e}{\mathrm{e}}

\newcommand{\KLD}{\mathrm{KLD}}
\newcommand{\Cov}{\mathrm{Cov}}
\newcommand{\E}{\mathds{E}}
\newcommand{\norm}[1]{\left\lVert#1\right\rVert}

\begin{document}

%

%

\twocolumn[

\aistatstitle{Resampling Base Distributions of Normalizing Flows}

\aistatsauthor{Vincent Stimper$^{1,2}$ \And Bernhard Schölkopf$^1$ \And  José Miguel Hernández-Lobato$^2$}
\runningauthor{Vincent Stimper, Bernhard Schölkopf, José Miguel Hernández-Lobato}

\aistatsaddress{$^1$Max Planck Institute for \\ Intelligent Systems, Germany \And  $^2$University of Cambridge, \\ United Kingdom} ]

\begin{abstract}
  Normalizing flows are a popular class of models for approximating probability distributions. However, their invertible nature limits their ability to model target distributions whose support have a complex topological structure, such as Boltzmann distributions. Several procedures have been proposed to solve this problem but many of them sacrifice invertibility and, thereby, tractability of the log-likelihood as well as other desirable properties. To address these limitations, we introduce a base distribution for normalizing flows based on learned rejection sampling, allowing the resulting normalizing flow to model complicated distributions without giving up bijectivity. Furthermore, we develop suitable learning algorithms using both maximizing the log-likelihood and the optimization of the Kullback-Leibler divergence, and apply them to various sample problems, i.e.\ approximating 2D densities, density estimation of tabular data, image generation, and modeling Boltzmann distributions. In these experiments our method is competitive with or outperforms the baselines.
\end{abstract}

\section{INTRODUCTION}

Inferring and approximating probability distributions is a central problem of unsupervised machine learning. A popular class of models for this task are normalizing flows \citep{Tabak2010,Tabak2013,Rezende2015}, which are given by an invertible map transforming a simple base distribution such as a Gaussian to obtain a complex distribution matching our target. Normalizing flows have been applied successfully to a variety of problems, such as image generation \citep{Dinh2015,Dinh2017,Kingma2018,Ho2019,Grcic2021}, audio synthesis \citep{Oord2018}, variational inference \citep{Rezende2015}, semi-supervised learning \citep{Izmailov2020} and approximating Boltzmann distributions \citep{Noe2019,Wu2020,Wirnsberger2020} among others \citep{Papamakarios2021}. However, with respect to some performance measures they are still outperformed by autoregressive models \citep{Chen2018,Parmar2018,Child2019}, generative adversarial networks (GANs) \citep{Gulrajani2017,Karras2019,Karras2020a,Karras2020}, and diffusion based models \citep{Sohl-Dickstein2015,Kingma2021}. One reason for this is an architectural limitation. Due to their bijective nature the normalizing flow transformation leaves the topological structure of the support of the base distribution unchanged and, since it is usually simple, there is a topological mismatch with the often complex target distribution \citep{Cornish2020}, thereby diminishing the modeling performance and even causing exploding inverses \citep{Behrmann2021}. Several solutions have been proposed, e.g.\ augmenting the space the model operates on \citep{Huang2020}, continuously indexing the flow layers \citep{Cornish2020}, and adding stochastic \citep{Wu2020} or surjective layers \citep{Nielsen2020}. However, these approaches sacrifice the bijectivity of the flow transformation, which means in most cases that the model is no longer tractable, memory savings during training are no longer possible \citep{Gomez2017}, and the model is no longer a perfect encoder-decoder pair. Some work has been done on using multimodal base distributions \citep{Izmailov2020,Ardizzone2020,Hagemann2021}, but the intention was to do classification or solve inverse problems with flow-based models and not to capture the inherent multimodal nature of the target distribution. \cite{Papamakarios2017} took a mixture of Gaussians as base distribution and showed that this can improve the performance.

In this work, we develop a method to obtain a more expressive base distribution through learned accept/reject sampling (LARS) \citep{Bauer2019}. It can be estimated jointly with the flow map by either maximum likelihood (ML) learning or minimizing the Kullback-Leibler (KL) divergence, matching the topological structure of the target's support. Moreover, we propose how the method can be scaled up to high dimensional datasets and demonstrate the effectiveness of our procedure on the tasks of learning 2D densities, estimating the density of tabular data, generating images, and the approximation of a 22 atom molecule's Boltzmann distribution.

\section{BACKGROUND}

\subsection{Normalizing Flows}
\label{sec:back_nf}

Let $z$ be a random variable taking values in $\mathds{R}^d$, having the density $p_\phi(z)$ parameterized by $\phi$. Furthermore, let $F_\theta:\mathds{R}^d \rightarrow \mathds{R}^d$ be a bijective map parameterized by $\theta$. We can compute the tractable density of the new random variable $x := F_\theta(z)$ with the change of variables formula
\begin{equation}
	p(x) = p_\phi(z) \left| \det(J_{F_\theta}(z))\right| ^{-1},
\end{equation}
where $J_{F_\theta}$ is the Jacobian matrix of $F_\theta$. This way of constructing a complex probability distribution $p(x)$ from a simple \emph{base distribution} $p(z)$ is called a \emph{normalizing flow}. We can use them to approximate a target density $p^*(x)$, which is done by optimizing a training objective. If the target density is unknown but samples from the corresponding distribution are available, we maximize the expected log-likelihood (LL) of the model
\begin{equation}
    \text{LL}(\theta, \phi) = \mathds{E}_{p^*(x)}\left[ \log\left( p(x)\right) \right].
    \label{equ:def_ll}
\end{equation}
Conversely, if the target density is given, we minimize the (reverse) KL divergence\footnote{For simplicity, we will call it just \emph{KL divergence} from now on.} \citep{Papamakarios2021}
\begin{equation}
    \KLD(\theta, \phi) := \mathds{E}_{p(x)}\left[ \log p(x) \right] - \mathds{E}_{p(x)}\left[ \log p^*(x) \right] ,
    \label{equ:def_rkld}
\end{equation}
or another difference measure for probability distributions such as the $\alpha$-divergence \citep{Hernandez-Lobato2016}.

\begin{figure}
	\centering
	\subfloat[]{
		\includegraphics[width=0.3\linewidth]{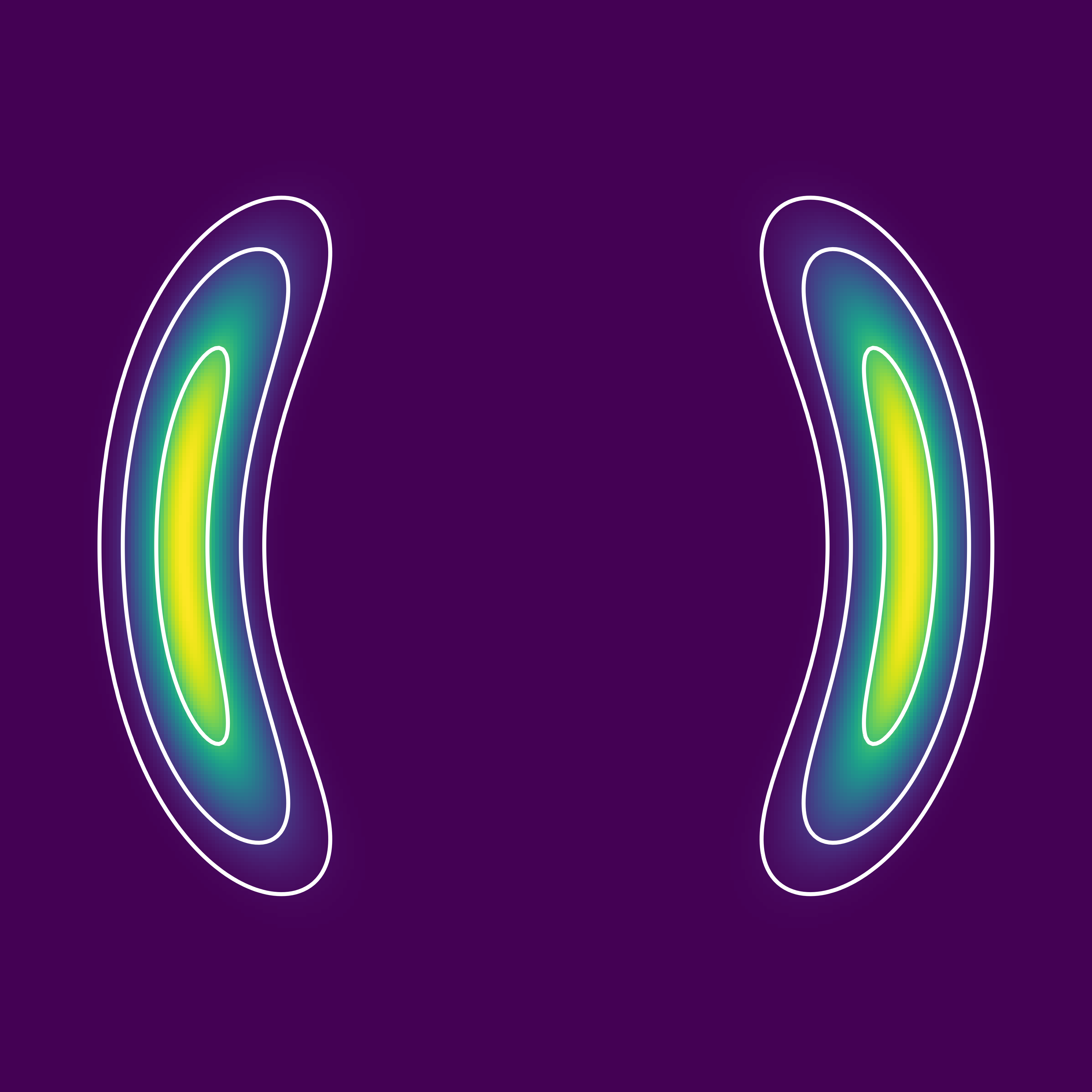}
	}
	\subfloat[]{
		\includegraphics[width=0.3\linewidth]{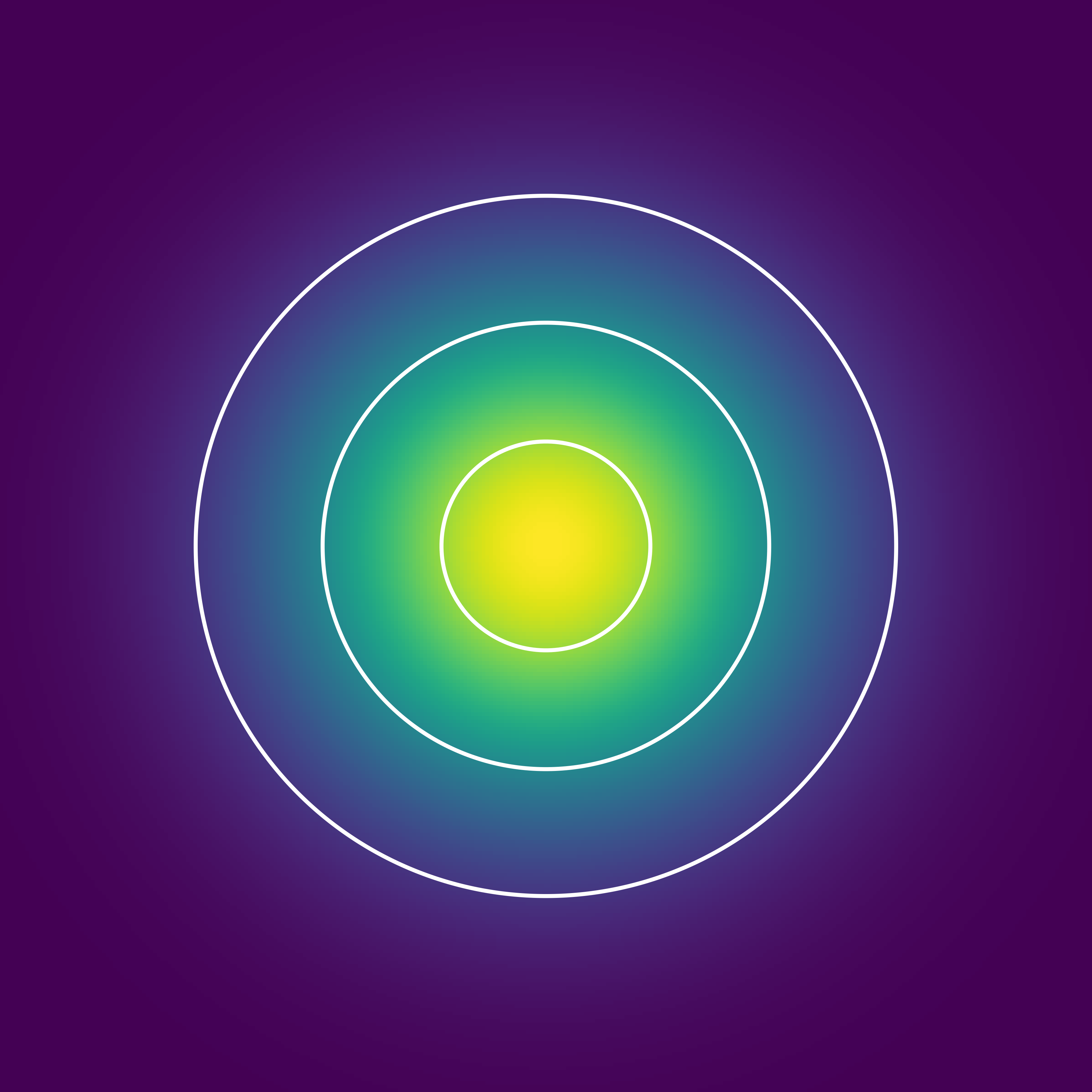}
	}
	\subfloat[]{
		\includegraphics[width=0.3\linewidth]{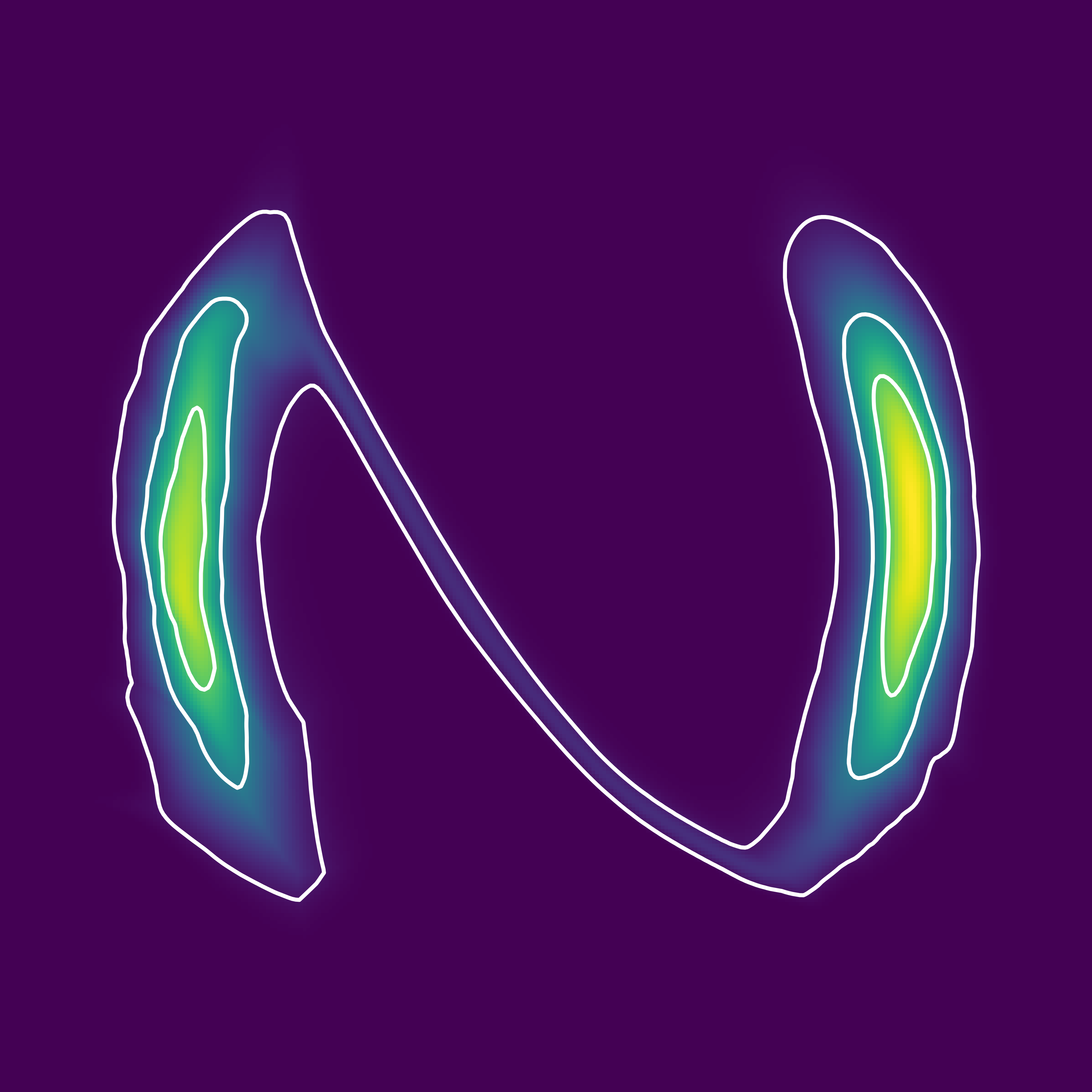}
	}
	\caption{Illustration of the architectural limitation of normalizing flows. (a) depicts the multimodal target distribution, (b) the Gaussian base distribution used, and (c) the learned real NVP model. The model's support has one connected component with a density filament between the modes.}
	\label{fig:artefacts_demo}
\end{figure}

To deal with high dimensional data, such as images, the multiscale architecture was introduced by \cite{Dinh2017}. As sketched in \autoref{fig:multscale_arch}, at the first level, the entire input $x$ is transformed by several flow layers $F_1$. The result is split up into two parts, $h_1^{(1)}$ and $h_1^{(2)}$. For images, this is typically done by first squeezing the image, i.e.\ reducing its height and width by a factor 2 and adding the surplus pixels as additional channels, and then splitting the resulting tensor along the channel dimension. $h_1^{(1)}$ is immediately factored out in the density, while $h_2^{(2)}$ is further transformed by the next set of flow layers $F_2$. The process is then repeated until a desired depth is reached. The full density for a multiscale architecture with $n$ levels is given by
\begin{equation}
	p(x) = \prod_{i=1}^{n} \left| \det\left( J_{F_i}(h_{i-1})\right) \right| p(z_i),
\end{equation}
where we set $h_0 = x$.

Normalizing flows can compete with other machine learning models on many benchmarks \citep{Papamakarios2021}. However, their performance is still impaired by an architectural weakness. The transformations defining a normalizing flow are invertible and such maps leave the topology of the sets they map unchanged \citep{Runde2005}. Consequently, the topological structure of the support of $p(z)$ is the same as that of $p(x)$. Usually, the base distribution is a Gaussian, which has only one mode, so its support consists of one connected component but the target distribution might be multimodal with the density between the modes being close to zero or even numerically zero due to finite precision so that the support consists of multiple disconnected components. As an exemplification we fit a real-valued non-volume preserving (real NVP) flow model with 8 coupling layers to a multimodal target distribution, see \autoref{fig:artefacts_demo}. The density of the trained model consists of one connected component covering the modes of the target, but connecting them via a density filament. Certain flow-based models, such as the residual flow \citep{Behrmann2019,Chen2019a}, can only converge to the target if they become non-invertible due to the topological mismatch \citep{Cornish2020}, thereby causing unstable training behaviour \citep{Behrmann2021}. Proposed solutions include increasing the model size significantly \citep{Chen2019a}, but this increases the computational cost and memory demand while the stability issues persist. Training can be stabilized via a suitable regularization, but this reduces the performance \citep{Behrmann2021}. Other approaches are discussed in Section \ref{sec:discussion}.

\subsection{Learned accept/reject sampling}
\label{sec:back_rejection_sampling}

Learned accept/reject sampling (LARS) is a method to approximate a $d$-dimensional distribution $q(z)$ by reweighting a proposal distribution $\pi(z)$ through a learned acceptance function $a_\phi: \mathds{R}^d \rightarrow [0,1]$, where $\phi$ are the learned parameters \citep{Bauer2019}. Given a sample $z_i$ from $\pi$, we will accept it with a probability $a_\phi(z_i)$, otherwise we reject it and draw a new sample until we accept one of the proposed samples. The resulting distribution is given by
\begin{equation}
	p_\infty(z) = \frac{\pi(z)a_\phi(z)}{Z};\quad\quad Z := \int\pi(z)a_\phi(z) \D z.
	\label{equ:lars_dens_inf}
\end{equation}
In order to limit the computational cost caused by high rejection rates, \cite{Bauer2019} introduced a truncation parameter $T\in\mathds{N}$. If the first $T-1$ samples from the proposal get rejected, we accept the $T^\text{th}$ sample no matter the value of the learned acceptance probability. Through this intervention, we alter the final sampling distribution to become
\begin{equation}
	p_T(z) = (1 - \alpha_T) \frac{a_\phi(z) \pi(z)}{Z} + \alpha_T \pi(z),
	\label{sec:lars_dens_trunc}
\end{equation}
where $\alpha_T := (1-Z)^{T-1}$, which reduces to \eqref{equ:lars_dens_inf} for $T\rightarrow\infty$. The integral \eqref{equ:lars_dens_inf} defining $Z$ is not tractable, so we cannot compute it directly. Instead, it is estimated via Monte Carlo sampling, i.e.
\begin{equation}
	Z \approx \frac{1}{S} \sum_{s = 1}^{S} a_\phi(z_s),
	\label{equ:lars_z_mc}
\end{equation}
where $z_s \sim \pi(z)$, which needs to be recomputed in every training iteration, as parameter changes in $a_\phi$ cause a change in $Z$.

LARS was first used to create a more expressive prior for variational autoencoders (VAEs) \citep{Kingma2014}, making it closer to the aggregate posterior distribution, thereby bringing the approximate posterior distribution closer to the ground truth. The resampled priors are trained jointly with the likelihood and the approximate posterior via maximization of the evidence lower bound. Since this only requires to evaluate the density of the prior at the data points, it is not even required to perform rejection sampling during training; therefore, the computational cost of training the whole model is only increased slightly.

\section{METHOD}
\subsection{Resampled base distributions}
\label{sec:method_basic_idea}

In Section \ref{sec:back_nf}, we argued that the topological structure of the support of the base distribution equals that of the overall flow distribution. To avoid artefacts resulting from mismatches between them, we aim to make the latter closer to the former. Therefore, we resample the base distribution with LARS, i.e.\ use it as our proposal so that its density becomes \eqref{sec:lars_dens_trunc}. Since there are no restrictions on the acceptance function $a_\phi$, we can use an arbitrarily complex neural network to model any desired topological structure. The resulting log-probability of the model is given by
\begin{equation}
    \begin{split}
        \log p(x) = & \log\pi(z) + \log\left( \alpha_T + (1 - \alpha_T)\frac{a_\phi(z)}{Z}\right) \\
        & - \log\left| \det J_{F_\theta}(z)\right| ,
    \end{split}
    \label{equ:flow_dist}
\end{equation}
where $F_\theta$ is the flow transformation, i.e.\ the composition of all flow layers, and $z = F_\theta^{-1}(x)$. In our case, the proposal is a Gaussian but it could be any other distribution or a more complicated model, such as a mixture of Gaussians or an autoregressive model. Depending on the application, $a_\phi$ will be a fully connected or a convolutional neural network, and details about how the architecture can be chosen are given in Appendix \ref{sec:a_architecture}. Since the evaluation of $a_\phi$ can be parallelized over the number of dimensions of the data $d$, we only add a constant computational overhead to our model. In contrast, autoregressive models scale linearly with $d$. We can sample from the model by performing LARS and propagating the accepted values through the flow map. The rejection rate, and hence the sampling speed, can be controlled via the truncation parameter $T$, which we set to $100$ in our experiments unless otherwise stated, but also through adding $Z$ to our loss function, which is discussed in Appendix \ref{sec:a_rejection_rate}.

Usually, the base distribution of normalizing flows has mean and variance parameters being trained with the flow layer parameters. Our proposal is simply a standard normal distribution, i.e.\ a diagonal Gaussian with mean zero and variance one. Thereby, we ensure that the samples from the proposal, which are the input for the neural network representing the learned acceptance probability $a_\phi$, come from a distribution which does not change during training. Instead, the mean and variance of the distribution can be altered after the resampling process by applying an affine flow layer with scale and shift being learnable parameters.

Note that while we retain the invertiblility of the flow, the probability distribution \eqref{equ:flow_dist} cannot be evaluated exactly since $Z$ needs to be estimated via \eqref{equ:lars_z_mc}. However, for large $T$ the base distribution reduces to \eqref{equ:lars_dens_inf} and, hence, we are only off by a constant meaning there would not be a bias when doing importance sampling, which is crucial for applications such as Boltzmann generators, see Section \ref{sec:boltz_gen}.

\subsection{Learning algorithms}
\label{sec:learning_algs}

The resampled base distribution can be trained jointly with the flow layers of our model. Both, the expected LL and the KL divergence, can be used as objectives. The former corresponds to maximizing \eqref{equ:def_ll}, which is done via stochastic gradient decent. As done by \cite{Bauer2019}, we sample from the proposal in each iteration to estimate the gradient of $Z$ with respect to the parameters, see \eqref{equ:lars_z_mc}. To stabilize training, we estimate the value of the normalization constant by an exponential moving average, see Appendix \ref{sec:ml_learn} for more details.

When the unnormalized target density $\hat{p}^*(x)$ is known, we can use the KL divergence \eqref{equ:def_rkld} as our objective. However, because sampling from the base distribution includes an acceptance/rejection step, we cannot apply the reparameterization trick \citep{Kingma2014} to obtain the gradients with respect to the model parameters. Instead, we derive an expression of the gradients of the KL divergence similar to that introduced by \cite{Grover2018b}.
\begin{theorem}
	\label{thm:rkld_grad}
	Let $p_\phi(z)$ be the base distribution of a normalizing flow, having parameters $\phi$, and $F_\theta$ be the respective invertible mapping, depending on its parameters $\theta$, such that the density of the model is
	\begin{equation}
		\log\left( p(x) \right) = \log\left( p_\phi(z) \right) - \log\left| \det J_{F_\theta}(z)\right| ,
		\label{equ:flow_with_param}
	\end{equation}
	with $x = F_\theta(z)$. Then, the gradients of the KL divergence with respect to the parameters are given by
	\begin{align}
	    \begin{split}
		\nabla_\phi \KLD(\theta, \phi) &= \Cov_{p_\phi(z)}\big\{ \nabla_\phi \log p_\phi(z), \\\log\left( p_\phi(z) \right) -& \log\left| \det J_{F_\theta}(z)\right|  - \log \hat{p}^*(F_\theta(z))\big\}
		\end{split}\label{equ:rkld_grad1}
		\\
		\begin{split}
		\nabla_\theta \KLD(\theta, \phi) &= -\E_{p_\phi(z)}\big[ \nabla_\theta \big( \log \hat{p}^*(F_\theta(z))  \\
		&\phantom{=-\E_{p_\phi(z)}\big[}+\log\left| \det J_{F_\theta}(z)\right| \big) \big]
		\end{split}\label{equ:rkld_grad2}
	\end{align}
\end{theorem}
The proof is given in Appendix \ref{sec:rkld_proof}. We will use \eqref{equ:rkld_grad1} and \eqref{equ:rkld_grad2} to compute the gradients of the KL divergence in our experiments and, thereby, demonstrate its effectiveness.

\subsection{Application to multiscale architecture}
\label{sec:method_multiscale}

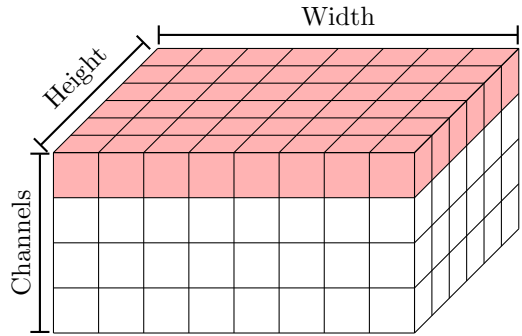
\begin{figure}
	\centering
	\begin{tikzpicture}[scale=0.6]
		\pgfmathsetmacro{\h}{6}
		\pgfmathsetmacro{\w}{8}
		\pgfmathsetmacro{\c}{4}
		
		\fill[red!30] (0, \c - 1, \h) rectangle (\w, \c, \h);
		\fill[red!30] (\w, \c, \h) -- (\w, \c, 0) -- (0, \c, 0) -- (0, \c, \h) -- (\w, \c, \h);
		\fill[red!30] (\w, \c, \h) -- (\w, \c, 0) -- (\w, \c - 1, 0) -- (\w, \c - 1, \h) -- (\w, \c, \h);
		
		\foreach \x in {0,...,\w} {
			\draw (\x ,0  ,\h ) -- (\x ,\c ,\h );
			\draw (\x ,\c ,\h ) -- (\x ,\c ,0  );
		}
		\foreach \x in {0,...,\c} {
			\draw (\w ,\x ,\h ) -- (\w ,\x ,0  );
			\draw (0  ,\x ,\h ) -- (\w ,\x ,\h );
		}
		\foreach \x in {0,...,\h} {
			\draw (\w ,0  ,\x ) -- (\w ,\c ,\x );
			\draw (0  ,\c ,\x ) -- (\w ,\c ,\x );
		}
	
		\draw[thick, |-|] (-0.3, 0, \h) -- (-0.3, \c, \h) node [midway, above, sloped] {Channels};
		\draw[thick, |-|] (-0.3, \c + 0.15, 0) -- (-0.3, \c + 0.15, \h) node [midway, above, sloped] {Height};
		\draw[thick, |-|] (0, \c + 0.3, 0) -- (\w, \c + 0.3, 0) node [midway, above, sloped] {Width};
	\end{tikzpicture}
	\caption{Visualization of a feature map when processing an image in a machine learning model. The unit which is used for factorization in our resampled base distribution is shown in red.}
	\label{fig:image_tensor_channel}
\end{figure}

LARS cannot be applied to very high dimensional distributions because we have to estimate $Z$ and its gradients via Monte Carlo sampling and the number of samples needed grows exponentially with the number of dimensions \citep{Bauer2019}. Although the base distribution of a normalizing flow must have the same number of dimensions as the target, we can reduce the number of dimensions significantly by factorization. Therefore, we extend the multiscale architecture, see Section \ref{sec:back_nf} and \citep{Dinh2017}, by further subdividing the base distribution at each level into factors with less than 100 dimensions. First, we squeeze the feature map until the product of height and width is smaller than 100. Then, each channel is treated as a separate factor, see \autoref{fig:image_tensor_channel}. To reduce the complexity of the model, we use parameter sharing to express the distribution of factors, i.e.\ there is one neural network per level with multiple outputs, each representing the acceptance probability $a_\phi$ for one channel. This also has the advantage that we can estimate the normalization constant and its gradient of all factors of one level in parallel by sampling from a Gaussian, passing the samples through the neural network and computing the average for each output dimension separately. As mentioned in Section \ref{sec:method_basic_idea}, the mean and variance is added via a constant coupling layer. Furthermore, the base distribution can be made class-conditional by making the mean and variance and/or $a_\phi$ dependent on the class. The latter can be efficiently achieved by adding more outputs to the neural network to have one value for $a_\phi$ per class and distribution if needed.

\section{EXPERIMENTS}

\subsection{2D distributions}
\label{sec:exp_2d}

\begin{figure*}[h]
    \centering
    \includegraphics[width=\linewidth]{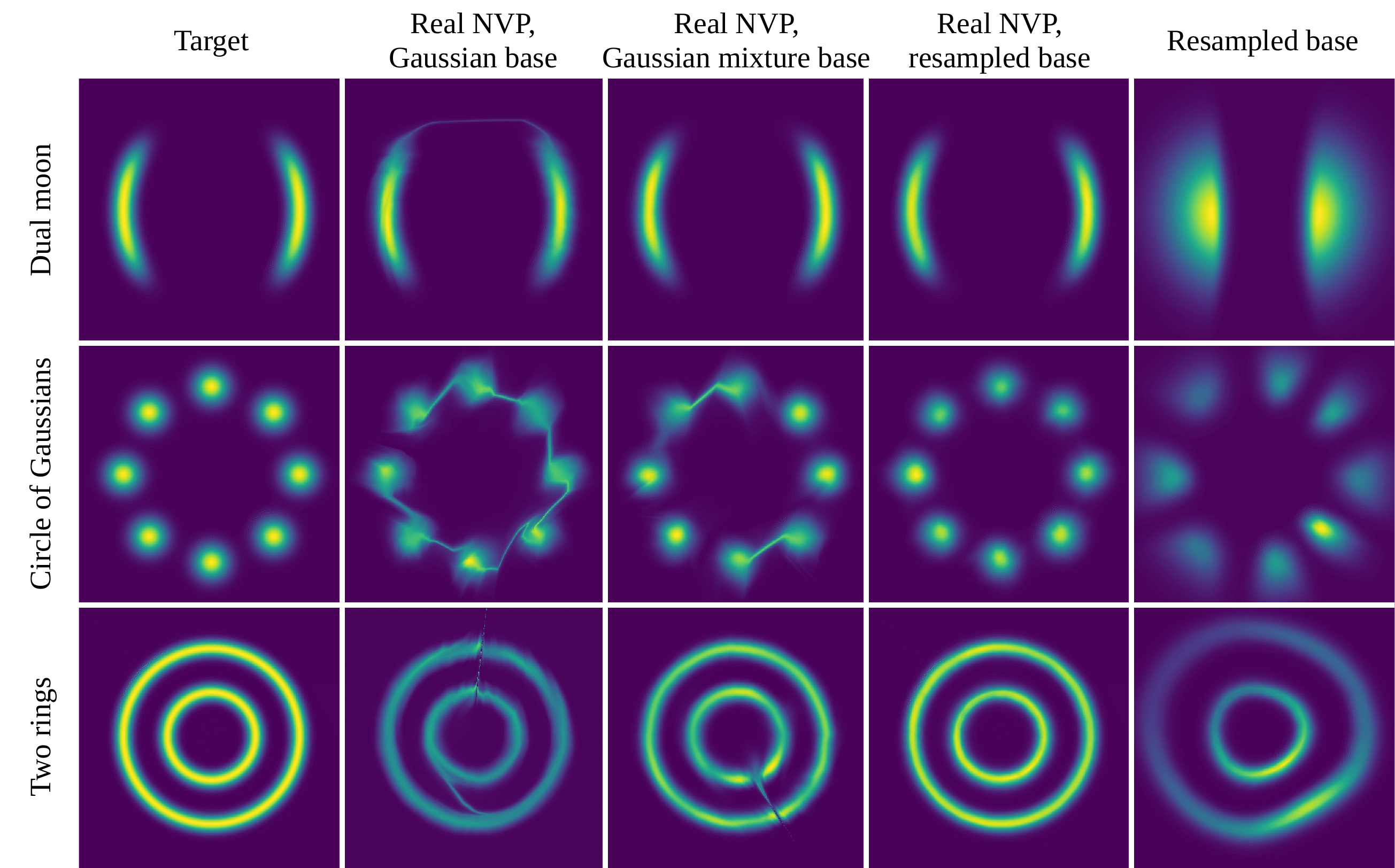}
    \caption{Visualization of the real NVP densities as well as the learned resampled base distribution when approximating three 2D distributions with complex topological structure. The models are trained via ML learning.}
    \label{fig:2d_rnvp_fkld}
\end{figure*}

\begin{table*}[h!]
  \caption{KL divergences of the target distribution and the flow models which are trained to approximate the three 2D distributions, shown in \autoref{fig:2d_rnvp_fkld}, with ML learning. For each target distribution and flow architecture, the model with the lowest KL divergence is marked in \textbf{bold}.}
  \label{tab:2d_fkld_kld}
  \centering
  \vspace{0.3cm}
  \begin{tabular}{l|lll|lll}
    Flow architecture & Real NVP & Real NVP & Real NVP & Residual & Residual & Residual \\
    Base distribution & Gaussian & Mixture & Resampled & Gaussian & Mixture & Resampled \\
    \hline 
    Dual moon & 1.83 & 1.80 & \textbf{1.77} & 1.82 & 1.80 & \textbf{1.76} \\
    Circle of Gaussians & 0.090 & 0.060 & \textbf{0.043} & 0.045 & 0.042 & \textbf{0.039} \\
    Two rings & 10.7 & 10.6 & \textbf{10.4} & 11.7 & 10.8 & \textbf{10.4}
  \end{tabular}
\end{table*}

\begin{table*}[t]
  \caption{LL on the test sets of the respective datasets of NSF, its CIF variant, and a NSF with a resampled base distribution (RBD). The values are averaged over 3 runs each and the standard error is given as a measure of uncertainty. The highest values within the confidence interval are marked in \textbf{bold}.}
  \label{tab:uci_ll}
  \centering
  \vspace{0.3cm}
  \begin{tabular}{l|llll}
    Method & Power & Gas & Hepmass & Miniboone \\
    \hline
    NSF & $\mathbf{0.69\pm0.00}$ & $13.01\pm0.02$ & $-14.30\pm0.05$ & $-10.68\pm0.06$ \\
    CIF-NSF & $\mathbf{0.68\pm0.01}$ & $13.08\pm0.00$ & $\mathbf{-13.83\pm0.10}$ & $-9.93\pm0.06$ \\
    RBD-NSF (ours) & $\mathbf{0.69\pm0.01}$ & $\mathbf{13.29\pm0.05}$ & $\mathbf{-14.02\pm0.12}$ & $\mathbf{-9.45\pm0.03}$
  \end{tabular}
\end{table*}

In this section, we aim to demonstrate that our method is indeed capable of modeling complicated distributions. Our code for all experiments is publicly available on GitHub at \url{https://github.com/VincentStimper/resampled-base-flows}.

We start with simple 2D distributions having supports with various topological structures, i.e.\ a distribution with two modes, one with eight modes, and one with two rings, see \autoref{fig:2d_rnvp_fkld} and \autoref{tab:equ_toy_examples}. We use both learning algorithms discussed in Section \ref{sec:learning_algs}. To train our flows via ML, we draw samples from our distributions via rejection sampling. As flow architectures, we choose real NVP \citep{Dinh2017} and residual flow \citep{Behrmann2019,Chen2019a} with 16 layers each. For each flow architecture, we train models with a Gaussian, a mixture of 10 Gaussians, and a resampled base distribution, having a Gaussian proposal and a neural network with 2 hidden layers with 256 hidden units each as well as a sigmoid output function as acceptance probability. 

The densities of the trained real NVP and residual flow models are show in the Figures \ref{fig:2d_rnvp_fkld} and \ref{fig:2d_residual_models}, respectively. With a Gaussian base distribution, the flows struggle to model the complex topological structure. For the trained real NVP models it is especially visible in \autoref{fig:2d_rnvp_fkld} that the density essentially consists of one connected component since there are density filaments between the modes and rings are not closed. The multimodal distributions can be fitted much better when using a mixture of Gaussians as base distribution, but especially the ring distribution can still not be represented properly. With a resampled base distribution the flow models the target distributions accurately without any artefacts. The base distributions assume the respective topological structure of the target while the flow transformation does the fine adjustment of the density. We also estimate the KL divergences of the target and the model distributions which are listed in \autoref{tab:2d_fkld_kld}. In all cases the flow model with the resampled base distribution outperforms the respective baselines.

Moreover, we train real NVP models with Gaussian and resampled base distributions with the KL divergence using the gradient estimators derived in \autoref{thm:rkld_grad}. The same architecture as the models trained with ML learning are used and their resulting densities are shown in \autoref{fig:2d_rnvp_rkld}. In addition, we also computed the KL divergences listed in \autoref{tab:2d_rkld_kld}. As for the previous experiments, the flow with the resampled base distribution clearly outperforms its baseline visually and quantitatively for all the three targets.

\begin{table}[h]
  \caption{KL divergences of the target distribution and the models which were trained using the KL divergence, shown in 
  \autoref{fig:2d_rnvp_rkld}. For each target distribution, the real NVP model with the lower KL divergences is marked in \textbf{bold}.}
  \label{tab:2d_rkld_kld}
  \centering
  \vspace{0.3cm}
  \begin{tabular}{l|ll}
    Base distribution & Gaussian & Resampled \\
    \hline
    Dual moon & 1.844 & \textbf{1.839} \\
    Circle of Gaussians & 0.167 & \textbf{0.122} \\
    Two rings & 11.5 & \textbf{10.3} 
  \end{tabular}
\end{table}

\subsection{Tabular data}

Next, we estimate the density of four tabular datasets from the UCI Machine Learning Repository \citep{Dheeru2021}. We use the same preprocessing and training, validation, and test splits as \cite{Papamakarios2017}, which have been adopted by others in the field \citep{Durkan2019,Cornish2020}. For each dataset, we train a Neural Spline Flow (NSF) \citep{Durkan2019}, its continuously indexed (CIF) variant \citep{Cornish2020}, and one with a resampled base distribution. The LL of the models are shown in \autoref{tab:uci_ll}. More details about the setup and the architecture as well as results for real NVP flows on the same datasets are given in \autoref{sec:tab_app}.

There is no significant performance difference of the three methods on the power dataset. On Hepmass, the resampled base distributions achieves similar performance to CIF, while both are better than the vanilla NSF. For the Gas and Miniboone dataset, the flow with the resampled base distribution clearly outperforms its baselines. When using real NVP, the difference is even larger on all datasets but Miniboone, as can be seen in \autoref{tab:uci_ll_rnvp}.

\subsection{Image generation}
\label{sec:exp_images}

To model images with our method, we train Glow \citep{Kingma2018} on the CIFAR-10 dataset \citep{Krizhevsky2009}. We use the multiscale architecture introduced in Section \ref{sec:method_multiscale}, where we compare a Gaussian with a respective resampled base distribution. As done by \cite{Kingma2018}, we use 3 levels, but train models with 8, 16, and 32 layers per level with each base distribution, with more details provided in \autoref{sec:app_image_generation}. For each model architecture, we do three seeded training runs and report bits per dimension on the test set in \autoref{tab:images_bpd}.

\begin{table}[h]
  \caption{Bits per dimension on the test set of the Glow models with Gaussian and resampled base distribution trained on CIFAR-10. For each architecture, three seeded training runs were done, the reported bits per dimension values are averages over these runs and the standard error is given as an uncertainty estimate. For each number of layers, the lowest values within the confidence interval is marked in \textbf{bold}.}
  \label{tab:images_bpd}
  \centering
  \vspace{0.3cm}
  \begin{tabular}{l|ll}
    Base distribution & Gaussian & Resampled \\
    \hline
    8 layers per level & $3.403\pm0.002$ & $\mathbf{3.399\pm0.001}$ \\
    16 layers per level & $3.339\pm0.001$ & $\mathbf{3.332\pm0.001}$ \\
    32 layers per level & $\mathbf{3.283\pm0.002}$ & $\mathbf{3.282\pm0.001}$ 
  \end{tabular}
\end{table}

\begin{figure*}[h]
    \centering
    \includegraphics[width=\linewidth]{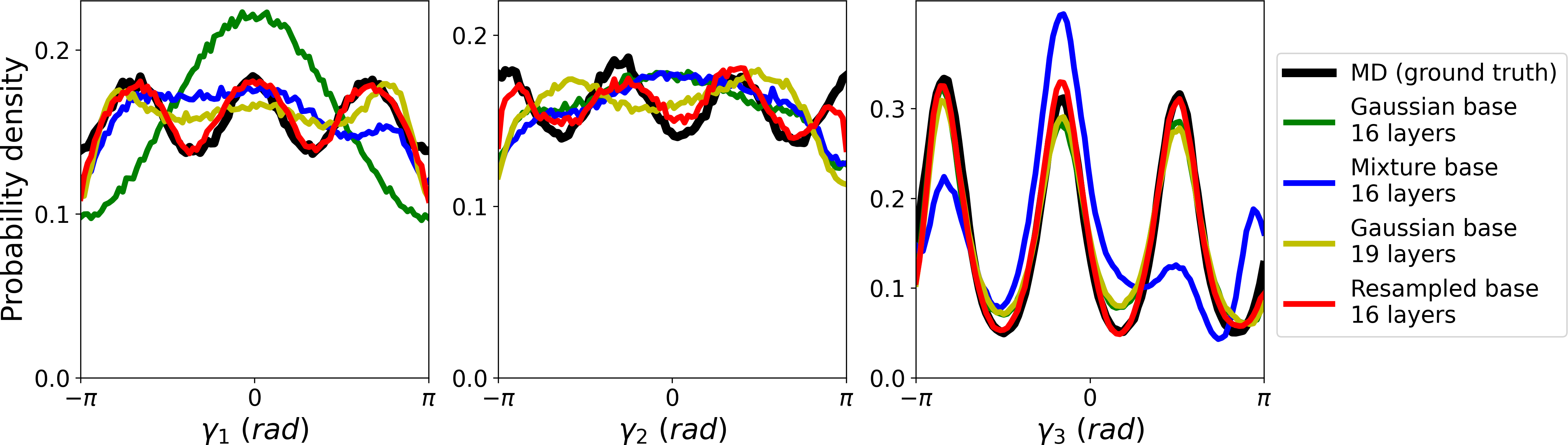}
    \caption{Marginal distribution of three dihedral angles of Alanine dipeptide. The ground truth was determined with a MD simulation. The flow models are based on real NVP and were trained via ML.}
    \label{fig:aldp_rnvp_fkld}
\end{figure*}

\begin{table*}[h]
  \caption{Quantitative comparison of the real NVP models approximating the Boltzmann distribution of Alanine dipeptide trained via ML learning. The LL is evaluated on a test set obtained with a MD simulation. The KL divergences of the 60 marginals were computed and the mean and median of them are reported. All results are averages over 10 runs, the standard error is given, and highers LL as well as lowest KL divergences are marked in \textbf{bold}.}
  \label{tab:aldp_kld_rnvp_fkld}
  \centering
  \vspace{0.3cm}
  \begin{tabular}{l|llll}
    Base distribution & Gaussian & Mixture & Gaussian & Resampled \\
    Number of layers & 16 & 16 & 19 & 16  \\
    \hline
    LL $(\times 10^2)$ & $1.8096\pm0.0002$ & $1.8106\pm0.0002$ & $1.8109\pm0.0001$ & $\mathbf{1.8118\pm0.0001}$ \\
    Mean KLD $(\times 10^{-3})$ & $1.76\pm0.08$ & $8.23\pm0.82$ & $1.35\pm0.03$ & $\mathbf{1.12\pm0.02}$ \\
    Median KLD $(\times 10^{-4})$ & $5.20\pm0.10$ & $43.5\pm6.0$ & $4.63\pm0.08$ & $\mathbf{4.36\pm0.05}$ 
  \end{tabular}
\end{table*}

The flow with the resampled base distribution outperforms the baseline when using 8 or 16 layers per level, while performing about equal with 32 layers. The difference is larger for smaller models, i.e.\ those where fewer layers are used, since models with many layers are already rather expressive. Using a more expressive base distribution also increases the model size and the training time, but this amounts only to 0.4-1.5\% and 5-15\%, respectively, versus a roughly linear increase with the number of layers. Hence, this can be a desirable trade-off, depending on the use case.

\subsection{Boltzmann generators}
\label{sec:boltz_gen}

An important application of normalizing flows is the approximation of Boltzmann distributions. Given the atom coordinates $x$ of a molecule, the likelihood of finding it in this state, i.e.\ the Boltzmann distribution, is proportional to $\mathrm{e}^{-u(x)}$, where $u$ denotes the energy of the system, which can be obtained through physical modeling. Usually, samples are drawn from this distribution through molecular dynamics (MD) simulations \citep{Leimkuhler2015}. However, the sampling process can be greatly accelerated by approximating the Boltzmann distribution with a normalizing flow, called a Boltzmann generator, and then sampling from the flow model \citep{Noe2019}.

Here, we approximate the Boltzmann distribution of the 22 atom Alanine dipeptide, which has been used as a benchmark system in the machine learning literature \citep{Wu2020,Campbell2021,Kohler2021}. We use the coordinate transformation introduced by \cite{Noe2019}, see also Appendix \ref{sec:coord_transform}, which incorporates the translational and rotational symmetry and reduces the number of dimensions from 66 to 60. Both ML learning and training using the KL divergence are used. For the former we generate a training dataset through a MD simulation over $10^7$ steps each and keep every $10^\text{th}$ sample, resulting in datasets with $10^6$ samples. With the same procedure we generate a test set to evaluate all trained models.

With ML learning we train real NVP models having 16 layers with a Gaussian, a mixture of 10 Gaussians, and a resampled base distribution. Furthermore, we train another real NVP model with a Gaussian base, but having 19 layers, which has roughly the same number of parameters as the real NVP model with the resampled base. More details of the architecture and the training procedure are listed in Appendix \ref{sec:aldp_setup}. The marginal distribution of three dihedral angles are shown in \autoref{fig:aldp_rnvp_fkld}.

Although we tried various methods of initializing the mixture of Gaussians, training it jointly with the flow turns out to be unstable leading to a poor fit of the marginals, which is especially visible for $\gamma_3$. Moreover, for two of the three angles, the 16-layered model with the Gaussian base distribution cannot represent the multimodal nature of the distribution accurately. Increasing the number of layers to 19 improves the result, but even this model is clearly outperformed by the real NVP with a resampled base distribution. To compare the performance quantitatively, we computed the LL on the test set and estimated the KL divergence between the MD samples and the models of the marginals through histograms for all 60 dimensions and report the mean and median. All performance measures where averaged over 10 seeded runs and are shown in \autoref{tab:aldp_kld_rnvp_fkld}. The real NVP model with the resampled base distribution outperforms all the baselines. The improved performance comes at the cost of increased training time, i.e.\ 49\% and 26\%, and sampling time, i.e.\ by a factor of 4 and 1.8, for the real NVP and the residual flow models with the resampled base distribution, when compared to their Gaussian counterparts. A further analysis of the Ramachandran plots of the models is done in Appendix \ref{sec:aldp_further_results}. There, we also do a comparison to stochastic normalizing flows \citep{Wu2020} and show the results of training residual flows with ML whereby the model with the resampled base distribution outperforms the baselines as well.

Moreover, we used the KL divergence to train real NVP models with Gaussian, mixture of Gaussians, and resampled base distributions as well, having the same architecture as the real NVP models with 16 layers in the experiments above. This is a challenging task since if samples from the model are too far away from the modes of the Boltzmann distribution, their gradients can be very high making training unstable. However, it is important for the application of Boltzmann generators since the necessity of creating a dataset through other expensive sampling procedures diminishes their ability to reduce the overall computational time needed for sampling. Details of the model architectures and the training procedure are given in Appendix \ref{sec:aldp_setup}. Although it involves rejection sampling, training the flow models with the resampled base distribution only took 15\% longer than the baseline models. As can be seen in \autoref{tab:aldp_kld_rnvp_rkld}, the real NVP model with a resampled base outperforms those with a Gaussian and a mixture of Gaussians; however, they are still inferior to flows trained via ML.

\section{DISCUSSION AND RELATED WORK}
\label{sec:discussion}

The main challenge we tackle in this work, i.e.\ that normalizing flows struggle to model distributions with supports having a complicated topological structure due to their invertible nature, has been addressed in several articles. \cite{Cornish2020} introduced a new set of variables for each flow layer, called continuous indices, which they used as additional input to the flow maps. Thereby, they relaxed the bijectivity of the transformation leading to a better model performance. \cite{Huang2020} augmented the dataset by auxiliary dimensions before applying their normalizing flow model. Although the topological constraints are still present in the augmented space, the marginal distribution of interest can be arbitrary complex. \cite{Wu2020,Nielsen2020} suggested adding sampling layers to the model. Hence, the topology of the support can be changed through the sampling process. \cite{Nielsen2020} also introduced surjective layers, which do not suffer from topological constraints and essentially combine VAEs with flow-based models. These approaches sacrifice the invertibility of the flow map, which has several disadvantages. First of all, the model is no longer a perfect autoencoder, i.e.\ the original datapoint cannot be fully recovered from its latent representation. Second, if the layers of the flow-based model are bijective, significant memory savings are possible \citep{Gomez2017}. Usually, when training layered models such as neural networks the activations of each layer need to be stored in the forward pass because they are needed for gradient computation in the backward pass. However, if the layers are invertible, the activations of the forward pass can be recomputed in the backward pass by applying the inverse of the layer to the activations of the previous layers. Thereby, models can be made basically infinitely deep  with a fixed memory budget. Thirdly, exact evaluation of the likelihood is no longer possible. To train the models, a bound needs to be derived which is optimized instead of the actual likelihood. Our model does not make this sacrifice since only the base distribution is altered, but the transformation of the normalizing flow model is still invertible. On the other side, the base distribution itself cannot be evaluated exactly because its normalization constant is unknown. It can be estimated via Monte Carlo sampling, but its logarithm, appearing in the LL of the model, is biased. However, as discussed in Section \ref{sec:method_basic_idea} for large truncation parameter $T$ we are only off by a constant so e.g. importance sampling could be done without a bias. Moreover, drawing samples from our model is less efficient as many samples from the proposal might get rejected before finally one is accepted and propagated through the flow.

An autoregressive base distribution was introduced by \cite{Bhattacharyya2020}. While they only considered image generation, their entire model, i.e.\ including the base distribution, is tractable in contrast to ours. However, the computational cost of their models scales with the square root of the number of pixels, while ours is constant. \cite{Izmailov2020,Ardizzone2020,Hagemann2021} explored normalizing flows with a multimodal base distribution, in their case a mixture of Gaussians. However, their intention was to model data with multiple classes, thereby performing classification and solving inverse problems. Our model allows to describe data with multiple classes as well through a conditional distribution, similar to the work of \cite{Dinh2017,Kingma2018}, but is also able to describe the complicated topological structure of the distribution of each class.

\cite{Bauer2019} used LARS successfully to create more expressive priors for VAEs, thereby boosting their performance. They demonstrated that the resampled prior can be learned jointly with the encoder and decoder by maximizing the evidence lower bound. In contrast, we showed that a resampled base distribution can be jointly trained with a normalizing flow transformation using both the LL and the KL divergence as an objective. For the latter we derived an expression of the gradient with reduced variance inspired by the work of \cite{Grover2018b}. Furthermore, \cite{Bauer2019} reported that they tried to fully factorize their resampled prior, which would allow them to scale to higher dimensional problems, but they were not able not beat the baseline of a VAE with a factorized Gaussian prior. We were successful by not fully factorizing our resampled base distribution, but defining factors for groups of variables. Moreover, combining LARS with the multiscale architecture of \cite{Dinh2017} and using a factorization similar to \citep{Ma2019} allowed us to scale up our base distribution even further. The largest base distribution in our work, used in Glow to model the CIFAR10 dataset, has 3072 dimensions, while the largest prior of \cite{Bauer2019} only had 100.

\section{CONCLUSION}

In this work, we introduced a base distribution for normalizing flows based on learned rejection sampling. We derived how it can be trained jointly with the flow layers maximizing the expected LL or minimizing the KL divergence. This base distribution can assimilate the complex topological structure of a target and, thereby, overcome a structural weakness of normalizing flows. By applying our procedure to 2D distributions, tabular data, images, and Boltzmann distributions we demonstrated that resampling the base distribution can improve their performance qualitatively and quantitatively.

\subsubsection*{Acknowledgements}
We thank Matthias Bauer, Richard Turner, Andrew Campbell, Austin Tripp, and David Liu for the helpful discussions.
Jos\'e Miguel Hern\'andez-Lobato acknowledges support from a Turing AI Fellowship under grant EP/V023756/1.
This work was supported by the German Federal Ministry of Education and Research (BMBF): Tübingen AI Center, FKZ: 01IS18039B; and by the Machine Learning Cluster of Excellence, EXC number 2064/1 - Project number 390727645. 

\bibliography{references}

\clearpage
\appendix

\thispagestyle{empty}

\onecolumn \makesupplementtitle

\section{LEARNING ALGORITHMS}
\label{sec:a_learn_alg}

\subsection{Estimating the normalization constant}
\label{sec:ml_learn}

To stabilize training, we use the exponential moving average to estimate the value of the normalization constant \citep{Bauer2019}. In practice, this means that if $Z_i$ is the current Monte Carlo estimate of the normalization constant, the exponential moving average $\langle Z \rangle_i$ is computed by
\begin{align}
	\langle Z \rangle_1 &= Z_1, \\
	\langle Z \rangle_i &= (1 - \epsilon)\langle Z \rangle_{i-1} + \epsilon Z_i \,\,\text{    for } i>1,
\end{align}
where $\epsilon$ is the decay parameter which we set to $0.05$ throughout this article. However, the gradients are estimated only with the current Monte Carlo estimate $Z_i$, because otherwise backpropagation through the entire history of $\langle Z \rangle_i$ would be necessary, which would be computationally expensive and memory demanding.

\subsection{Gradient estimators of the Kullback-Leibler divergence}
\label{sec:rkld_proof}

We repeat \autoref{thm:rkld_grad} as stated in the main text and supplement its proof.
\setcounter{theorem}{0}
\begin{theorem}
	\label{thm:rkld_grad_}
	Let $p_\phi(z)$ be the base distribution of a normalizing flow, having parameters $\phi$, and $F_\theta$ be the respective invertible mapping, depending on its parameters $\theta$, such that the density of the model is
	\begin{equation}
		\log\left( p(x) \right) = \log\left( p_\phi(z) \right) - \log\left| \det J_{F_\theta}(z)\right| ,
		\label{equ:flow_with_param_}
	\end{equation}
	with $x = F_\theta(z)$. Then, the gradients of the KL divergence with respect to the parameters are given by
	\begin{align}
		\nabla_\phi \KLD(\theta, \phi) &= \Cov_{p_\phi(z)}\big\{ \log\left( p_\phi(z) \right) - \log\left| \det J_{F_\theta}(z)\right|  - \log \hat{p}^*(F_\theta(z)), \nabla_\phi \log p_\phi(z)\big\} \label{equ:rkld_grad1_}\\
		\nabla_\theta \KLD(\theta, \phi) &= -\E_{p_\phi(z)}\left[ \nabla_\theta \big( \log\left| \det J_{F_\theta}(z)\right|  + \log \hat{p}^*(F_\theta(z))\big) \right] \label{equ:rkld_grad2_}
	\end{align}
\end{theorem}
\begin{proof}
    The KL divergence is defined as
    \begin{equation}
        \KLD(\theta, \phi) := \mathds{E}_{p(x)}\left[ \log p(x) \right] - \mathds{E}_{p(x)}\left[ \log p^*(x) \right] .
        \label{equ:def_rkld_}
    \end{equation}
	By plugging in \eqref{equ:flow_with_param_} into \eqref{equ:def_rkld_} we obtain
	\begin{equation}
		\KLD(\theta, \phi) = \E_{p_\phi(z)}\left[ \log p_\phi(z) - \log\left| \det J_{F_\theta}(z)\right|  - \log \hat{p}^*(F_\theta(z)) \right] .
		\label{equ:rkld_proof}
	\end{equation}
	Computing the gradient of \eqref{equ:rkld_proof} with respect to $\theta$ is straight forward.
	\begin{equation}
	\begin{split}
		\nabla_\theta \KLD(\theta, \phi) & = \nabla_\theta\E_{p_\phi(z)}\left[ \log p_\phi(z) - \log\left| \det J_{F_\theta}(z)\right|  - \log \hat{p}^*(F_\theta(z)) \right] \\ 
		& = \E_{p_\phi(z)}\left[ \nabla_\theta\big( \log p_\phi(z) - \log\left| \det J_{F_\theta}(z)\right|  - \log \hat{p}^*(F_\theta(z)) \big)\right] \\
		& = - \E_{p_\phi(z)}\left[ \nabla_\theta\big(\log\left| \det J_{F_\theta}(z)\right|  + \log \hat{p}^*(F_\theta(z)) \big)\right]
	\end{split}
	\end{equation}
	To get the gradient with respect to $\phi$, we decompose \eqref{equ:rkld_proof} into two parts and consider their gradients separately.
	\newpage
	\begin{equation}
	\begin{split}
		\nabla_\phi \E_{p_\phi(z)}\left[ \log p_\phi(z) \right] & = \nabla_\phi \int  p_\phi(z)\log p_\phi(z) \D z \\
		& = \int \nabla_\phi \big( p_\phi(z)\log p_\phi(z) \big) \D z \\ 
		& = \int \nabla_\phi p_\phi(z) + \log p_\phi(z) \nabla_\phi p_\phi(z) \D z \\ 
		& = \nabla_\phi \underbrace{\int p_\phi(z) \D z}_{=1} + \int p_\phi(z)\log p_\phi(z) \nabla_\phi \log p_\phi(z) \D z \\ 
		& = \E_{p_\phi(z)}\left[ \log p_\phi(z) \nabla_\phi \log p_\phi(z) \right] \\ 
	\end{split}
	\end{equation}
	\begin{equation}
	    \begin{split}
	    \mathrm{ld} :&= \log\left| \det J_{F_\theta}(z)\right| + \log \hat{p}^*(F_\theta(z)) \\
		\nabla_\phi \E_{p_\phi(z)}\left[ \mathrm{ld}\, \right] & = \nabla_\phi \int \mathrm{ld}\, p_\phi(z) \D z \\
		& = \int \mathrm{ld}\, \nabla_\phi p_\phi(z) \D z \\
		& = \int \mathrm{ld}\, p_\phi(z) \nabla_\phi \log p_\phi(z) \D z \\
		& = \E_{p_\phi(z)}\left[ \mathrm{ld}\, \nabla_\phi \log p_\phi(z) \right] 
	    \end{split}
	\end{equation}
	Using these two expressions, we obtain
	\begin{align}
		\nabla_\phi \KLD(\theta, \phi) & = \E_{p_\phi(z)}\left[ \big( \log p_\phi(z) - \log\left| \det J_{F_\theta}(z)\right| - \log \hat{p}^*(F_\theta(z)) \big) \nabla_\phi \log p_\phi(z) \right] \label{equ:rkld_grad_phi_l1} \\
		& = \Cov_{p_\phi(z)}\big\{ p_\phi(z) - \log\left| \det J_{F_\theta}(z)\right|  - \log \hat{p}^*(F_\theta(z)), \nabla_\phi \log p_\phi(z)\big\}. \label{equ:rkld_grad_phi_l2}
	\end{align}
	When concluding \eqref{equ:rkld_grad_phi_l2} from \eqref{equ:rkld_grad_phi_l1} we used the well known identity
	\begin{equation}
	    \begin{split}
	        \E_{p_\phi(z)}\left[ \nabla_\phi \log p_\phi(z)\right] &= \int p_\phi(z)\nabla_\phi \log p_\phi(z) \D z \\ 
	        &= \int \frac{p_\phi(z)}{p_\phi(z)}\nabla_\phi p_\phi(z) \D z = \nabla_\phi \underbrace{\int p_\phi(z) \D z}_{=1} = 0.
	    \end{split}
	\end{equation}
\end{proof}

\section{MULTISCALE ARCHITECTURE}

As already mentioned in the main paper, \cite{Dinh2017} introduced the multiscale architecture for normalizing flows to deal with high dimensional data such as images. As sketched in \autoref{fig:multscale_arch}, initially, the entire input $x$ is transformed by several flow layers. The result is split up into two parts, $h_1^{(1)}$ and $h_1^{(2)}$. \cite{Dinh2017} did this by first squeezing the image, i.e. reducing the height and width of the image by a factor 2 and adding the surplus pixels as additional channels, and then splitting the resulting tensor along the channel dimension. $h_1^{(1)}$ is immediately factored out in the density, while $h_2^{(2)}$ is further transformed by $F_2$. The process is the repeated until a desired depth is reached. The output of the last map, in \autoref{fig:multscale_arch} it is $F_4$, is not split, but directly passed to its base distribution. The full density for a multiscale architecture with $n$ levels is given by
\begin{equation*}
	p(x) = \prod_{i=1}^{n} \left| \det\left( J_{F_i}(h_{i-1})\right) \right| p(z_i),
\end{equation*}
where we set $h_0 = x$.

\begin{figure}[h]
	\centering
	\begin{tikzpicture}
		\fill[blue!30, rounded corners] (0, 0) rectangle (1, 8) node[pos=.5, black] {$x$};
		\fill[red!30, rounded corners] (2.5, 0) rectangle (3.5, 8);
		\draw[dotted, thick] (2.5, 4) -- (3.5, 4);
		\fill[red!30, rounded corners] (5, 0) rectangle (6, 4);
		\draw[dotted, thick] (5, 2) -- (6, 2);
		\fill[red!30, rounded corners] (7.5, 0) rectangle (8.5, 2);
		\draw[dotted, thick] (7.5, 1) -- (8.5, 1);
		\fill[yellow!30, rounded corners] (10, 0) rectangle (11, 0.95);
		\fill[yellow!30, rounded corners] (10, 1.05) rectangle (11, 2);
		\fill[yellow!30, rounded corners] (10, 2.1) rectangle (11, 4);
		\fill[yellow!30, rounded corners] (10, 4.1) rectangle (11, 8);
		\node at (3, 2) {$h_1^{(2)}$};
		\node at (3, 6) {$h_1^{(1)}$};
		\node at (5.5, 1) {$h_2^{(2)}$};
		\node at (5.5, 3) {$h_2^{(1)}$};
		\node at (8, 0.5) {$h_3^{(2)}$};
		\node at (8, 1.5) {$h_3^{(1)}$};
		\node at (10.5, 6) {$z_1$};
		\node at (10.5, 3) {$z_2$};
		\node at (10.5, 1.5) {$z_3$};
		\node at (10.5, 0.5) {$z_4$};
		\draw[->, thick] (1.1, 4) -- (2.4, 4) node [midway, above] {$F_1$};
		\draw[->, thick] (3.6, 2) -- (4.9, 2) node [midway, above] {$F_2$};
		\draw[->, thick] (6.1, 1) -- (7.4, 1) node [midway, above] {$F_3$};
		\draw[->, thick] (8.6, 0.5) -- (9.9, 0.5) node [midway, above] {$F_4$};
		\draw[->, thick, dashed] (3.6, 6) -- (9.9, 6);
		\draw[->, thick, dashed] (6.1, 3) -- (9.9, 3);
		\draw[->, thick, dashed] (8.6, 1.5) -- (9.9, 1.5);
	\end{tikzpicture}
	\caption{Multiscale architecture with four levels as introduced in \citep{Dinh2017}. First, the entire input $x$ is transformed by $F_1$. The result is then split up into two parts of which one of them is factored out immediately and the other one is further processed by $F_2$. This process is repeated a few times until the desired depth is reached. The input is drawn in blue, intermediate results are red, and the components of the final variable $z$ are yellow.}
	\label{fig:multscale_arch}
\end{figure}
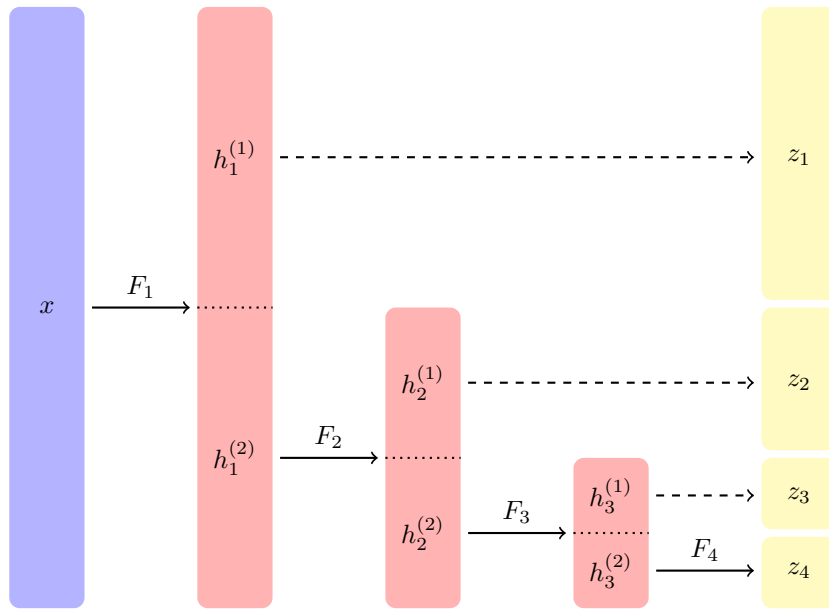

\section{LEARNED ACCEPTANCE PROBABILITY}

\subsection{Choosing the architecture}
\label{sec:a_architecture}

In order to get an impression of what the architecture of the neural network defining the acceptence probability $a$ for LARS, we did an ablation experiment on the Power UCI dataset. We left the flow architecture of a real NVP model constant but changed the number of hidden layers and units of the neural network representing $a$. The baseline model with a Gaussian base distribution achieved $0.330\pm0.003$ on the test set. When changing the number of hidden layers we used 512 hidden units and 3 hidden layers when changing the number hidden units.

\begin{table}[h]
    \centering
    \caption{LL of the test set for different number of hidden layers for $a$ while leaving the number of hidden units constant at 512.}
    \label{tab:abl_hl}
    \begin{tabular}{l|lllll}
         Hidden layers & 1 & 3 & 5 & 7 & 9 \\
         \hline
         LL & 0.37 & 0.53 & 0.58 & 0.62 & 0.63 
    \end{tabular}
\end{table}

\begin{table}[h]
    \centering
    \caption{LL of the test set for different number of hidden layers for $a$ while leaving the number of hidden layers constant at 3.}
    \label{tab:abl_hl}
    \begin{tabular}{l|lllll}
         Hidden units & 32 & 128 & 512 & 2048 & 8192 \\
         \hline
         LL & 0.39 & 0.45 & 0.53 & 0.61 & 0.61
    \end{tabular}
\end{table}

We see that both the number of hidden layers and units is important. The LL increases as we are adding more with diminishing returns. However, note that especially inceasing the number of hidden units increases the parameter count as well as the computational cost; hence, an application specific trade-off needs to be found.

\subsection{Tuning the rejection rate}
\label{sec:a_rejection_rate}

\begin{figure}[h]
    \centering
    \includegraphics[width=0.7\linewidth]{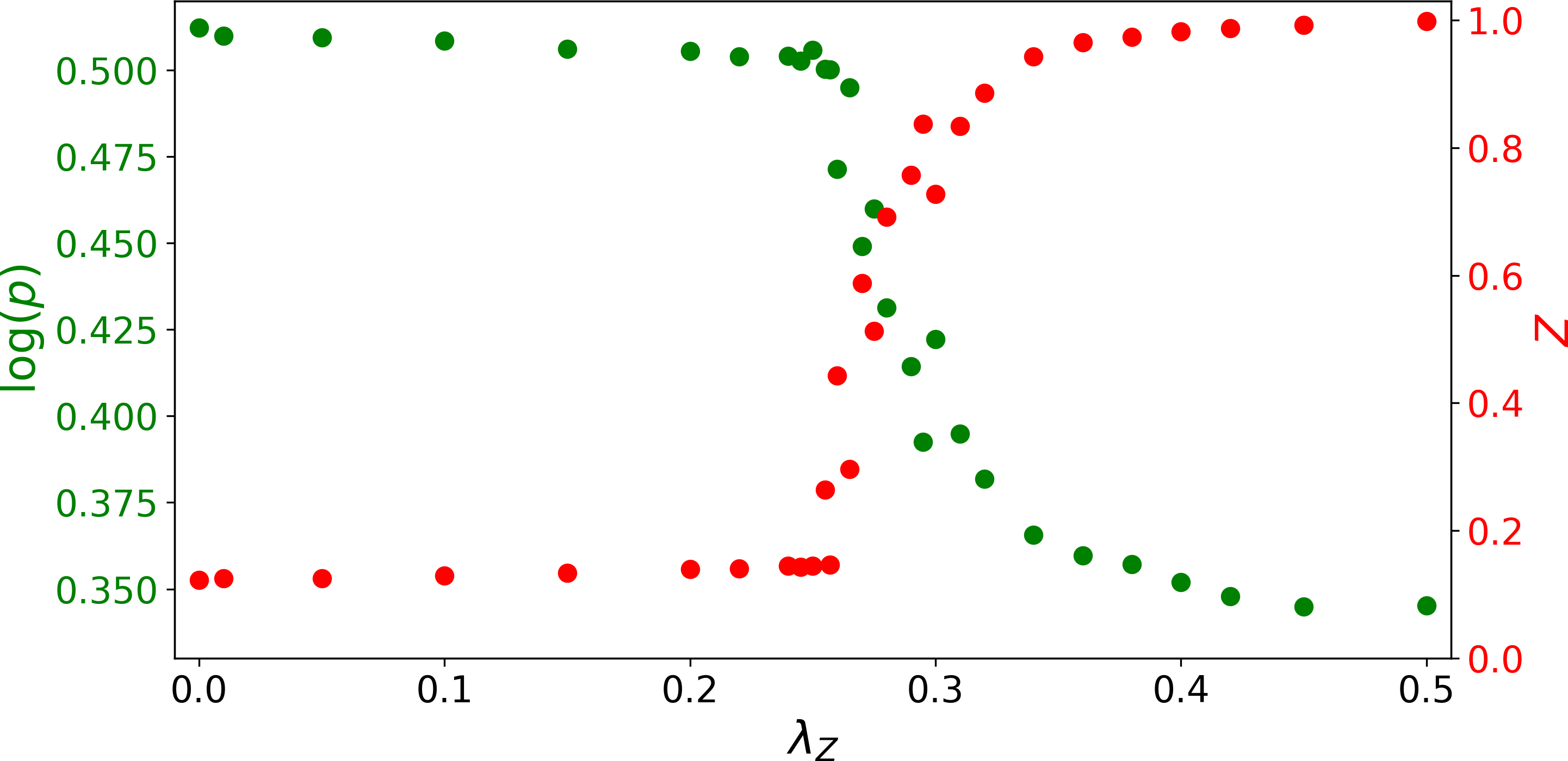}
    \caption{LL on the test set and $Z$ with respect to the hyperparameter $\lambda_Z$ introduced in \eqref{equ:loss_logp_z}.}
    \label{fig:power_logp_z}
\end{figure}

As discussed in the main text, the rejection rate of LARS can be controlled through the truncation parameter $T$. It sets a limit on how often subsequent proposals can be rejection in order to generate one sample. However, it does not tell us something about the actual rejection rate determining the sampling speed, which might be lower. The number of expected samples per sample from the proposal $\pi$ is given by
\begin{equation}
    \mathds{E}_\pi(z)[a(z)] = \int a(z)\pi(z)\D z = Z,
\end{equation}
which is equivalent to the normalization constant $Z$. Hence, if we increase $Z$ we can decrease the rejection rate. We can simply do so by including it in our optimization, e.g. when doing ML we can instead minimize the loss
\begin{equation}
    \mathcal{L} = -\mathds{E}_{p^*(x)}[\log p(x)] - \lambda_Z Z,
    \label{equ:loss_logp_z}
\end{equation}
where $\lambda_Z \in \mathds{R}_+$ is a positive hyperparameter.

In order to test this procedure, we trained 30 real NVP models with a resampled base distribution with different values of $\lambda_Z$ on the UCI Power dataset. The neural network representing the acceptance probability $a$ had 3 hidden layers with 512 hidden units and we set $T=20$. In \autoref{fig:power_logp_z} we show the LL of the models on the test set as well as $Z$ depending on the hyperparameter $\lambda_Z$. We see that by increasing $\lambda_Z$ we can trade off performance in terms of LL with the expected number of LARS samples per sample from the proposal. When $Z$ approaches one, i.e. nearly all samples from the proposal get accepted, the LL drops to the value achieved by the flow with a Gaussian base distribution being $0.330\pm0.003$, see \autoref{tab:uci_ll_rnvp}.

\section{2D DISTRIBUTIONS}

The densities of the distributions used as sample targets in Section \ref{sec:exp_2d} are given in \autoref{tab:equ_toy_examples}.

\begin{table}[h!]
	\caption{Logarithm of the unnormalized densities of the target distributions used in Section \ref{sec:exp_2d}.}
	\label{tab:equ_toy_examples}
	\centering
	\vspace{0.3cm}
	\begin{tabular}{c|c}
		& Unnormalized log density \\
		\hline 
		& \\[-1em]
		Dual Moon & $\displaystyle{-\frac{\left( \norm{z} - 1\right) ^2}{0.08} - \frac{\left( \left| z_1\right| - 2\right) ^2}{0.18} + \log\left( 1 + \e^{-\frac{4z_1}{0.09}}\right) }$ \\[1em]
		& \\[-1em]
		Circle of Gaussians &  $\displaystyle{\log\left[ \sum_{i=1}^8 \left( \frac{9}{2\pi\left( 2 - \sqrt{2}\right) } \e^{-\frac{9\left( \left( z_1 - 2\sin\left( \frac{2\pi}{8i}\right) \right) ^ 2 + \left( z_1 - 2\cos\left( \frac{2\pi}{8i}\right) \right) ^ 2 \right) }{4 - 2\sqrt{2}}}\right)\right]  }$ \\[1,5em]
		& \\[-1em]
		Two Rings & $\displaystyle{\log\left[ \sum_{i=1}^2 \left( \frac{32}{\pi} \e^{-32\left(\norm{z} - i - 1 \right) ^2 }\right)\right]  }$ \\[1em]
	\end{tabular}
\end{table}

\begin{figure}[!h]
    \centering
    \includegraphics[width=0.95\linewidth]{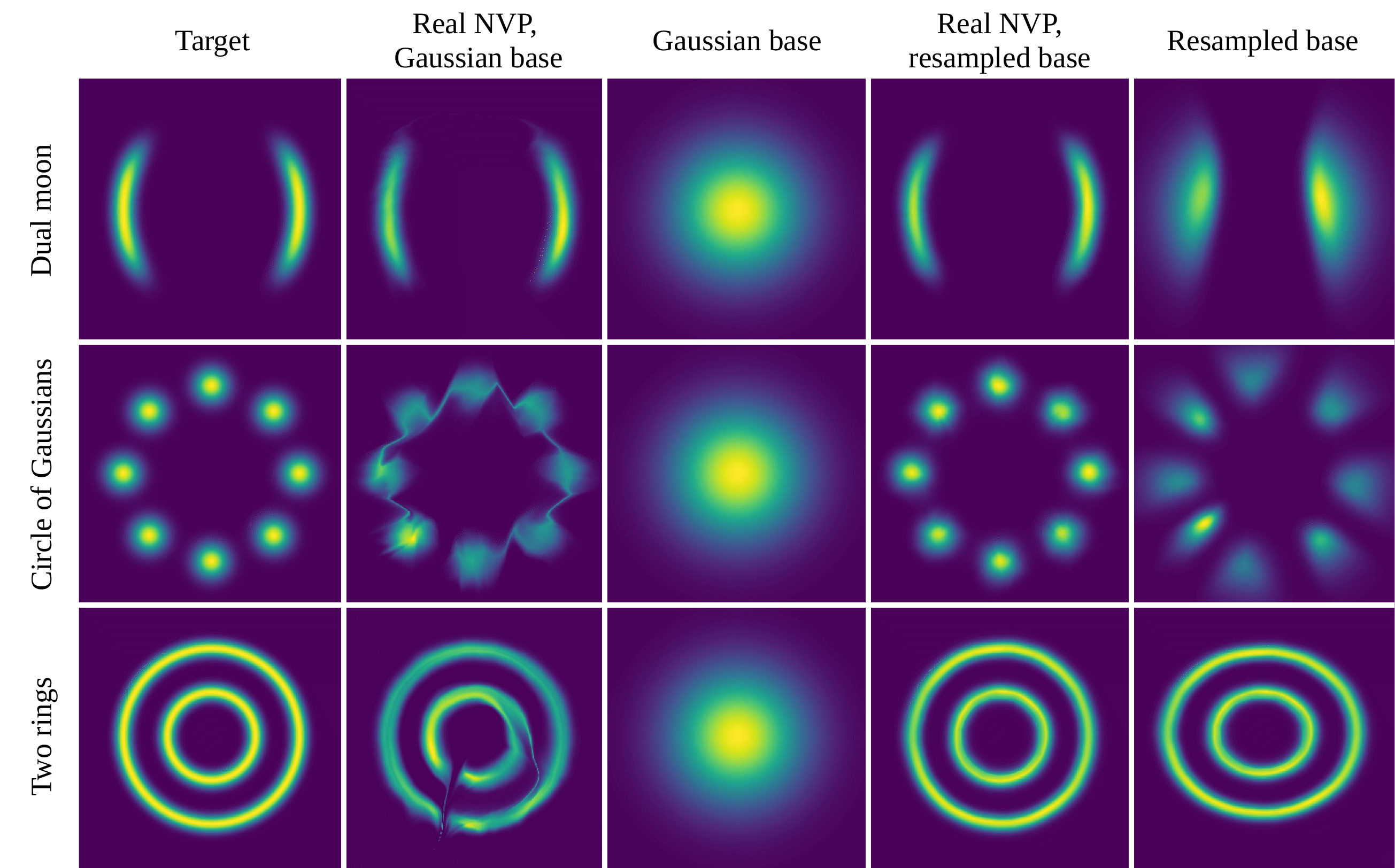}
    \caption{Visualization of the densities when approximating three 2D distributions with complex topological structure. Real NVP models with Gaussian and a resampled base distributions where trained using the KL divergence.}
    \label{fig:2d_rnvp_rkld}

    \includegraphics[width=0.75\linewidth]{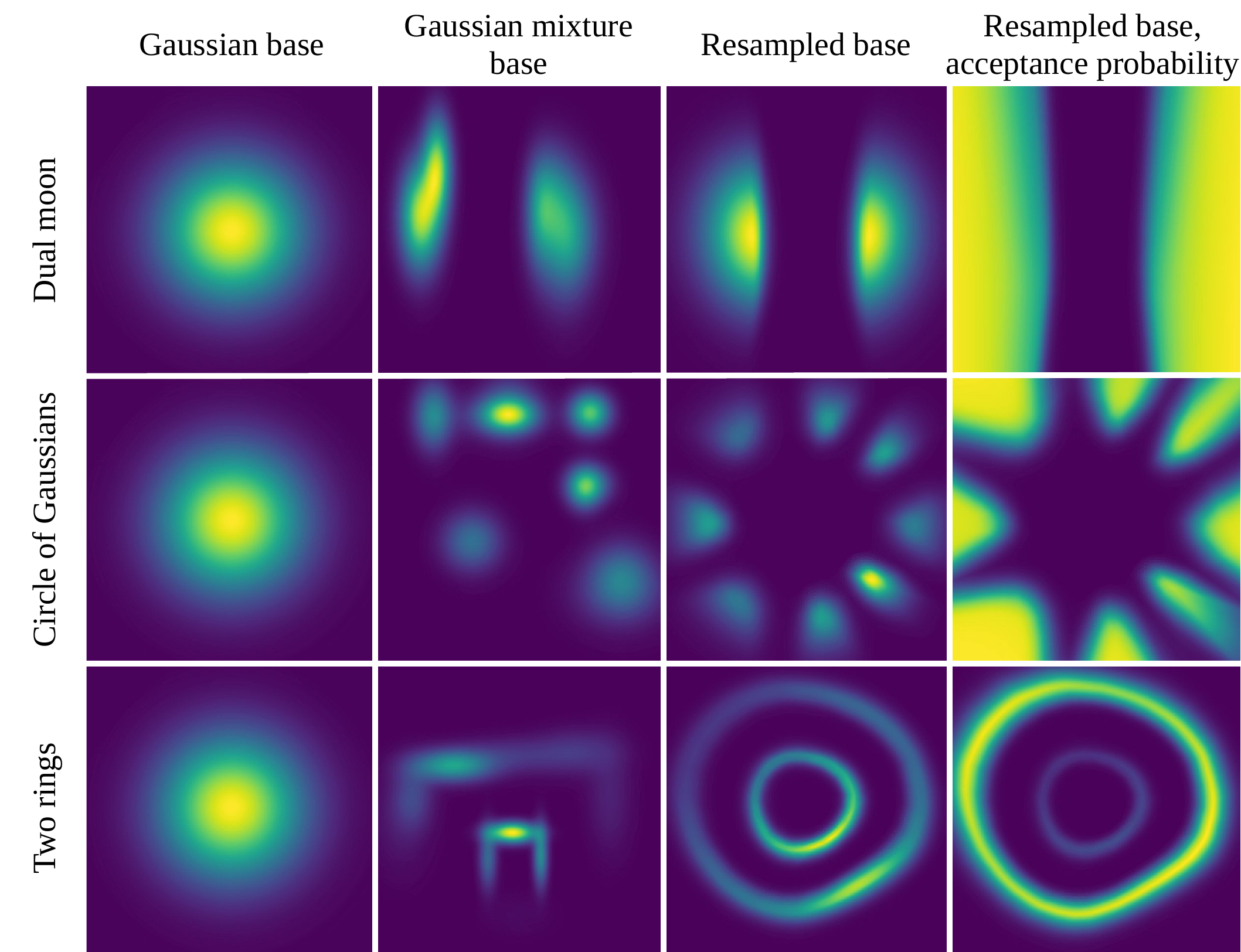}
    \caption{Visualization of the learned base distributions of the real NVP flow models shown in \autoref{fig:2d_rnvp_fkld}.}
    \label{fig:2d_residual_base}
\end{figure}

\begin{figure}[!h]
    \centering
    \includegraphics[width=0.75\linewidth]{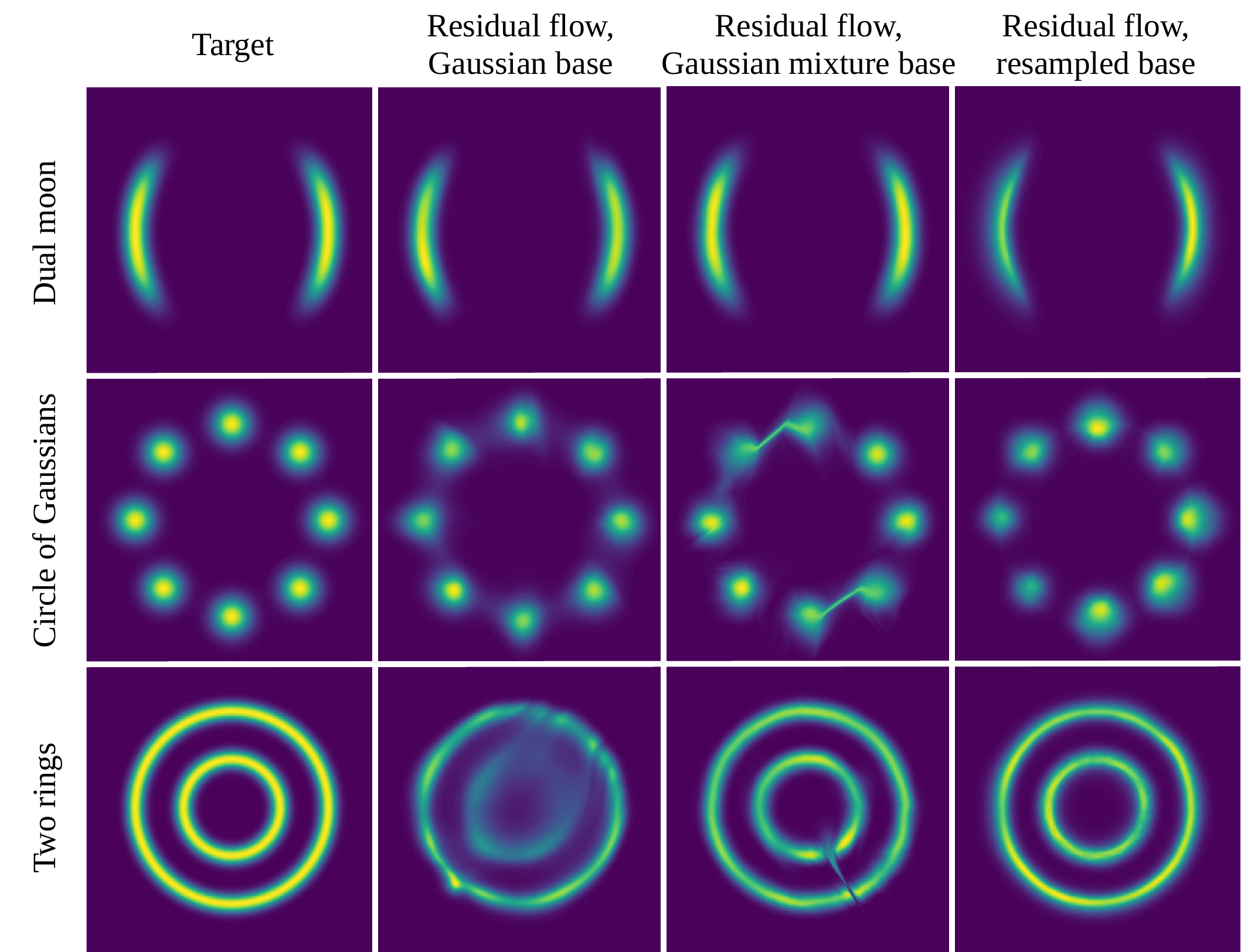}
    \caption{Visualization of the residual flow densities when approximating three 2D distributions with complex topological structure. The models were trained using ML learning and the corresponding base distributions are shown in \autoref{fig:2d_residual_base}.}
    \label{fig:2d_residual_models}

    \includegraphics[width=0.75\linewidth]{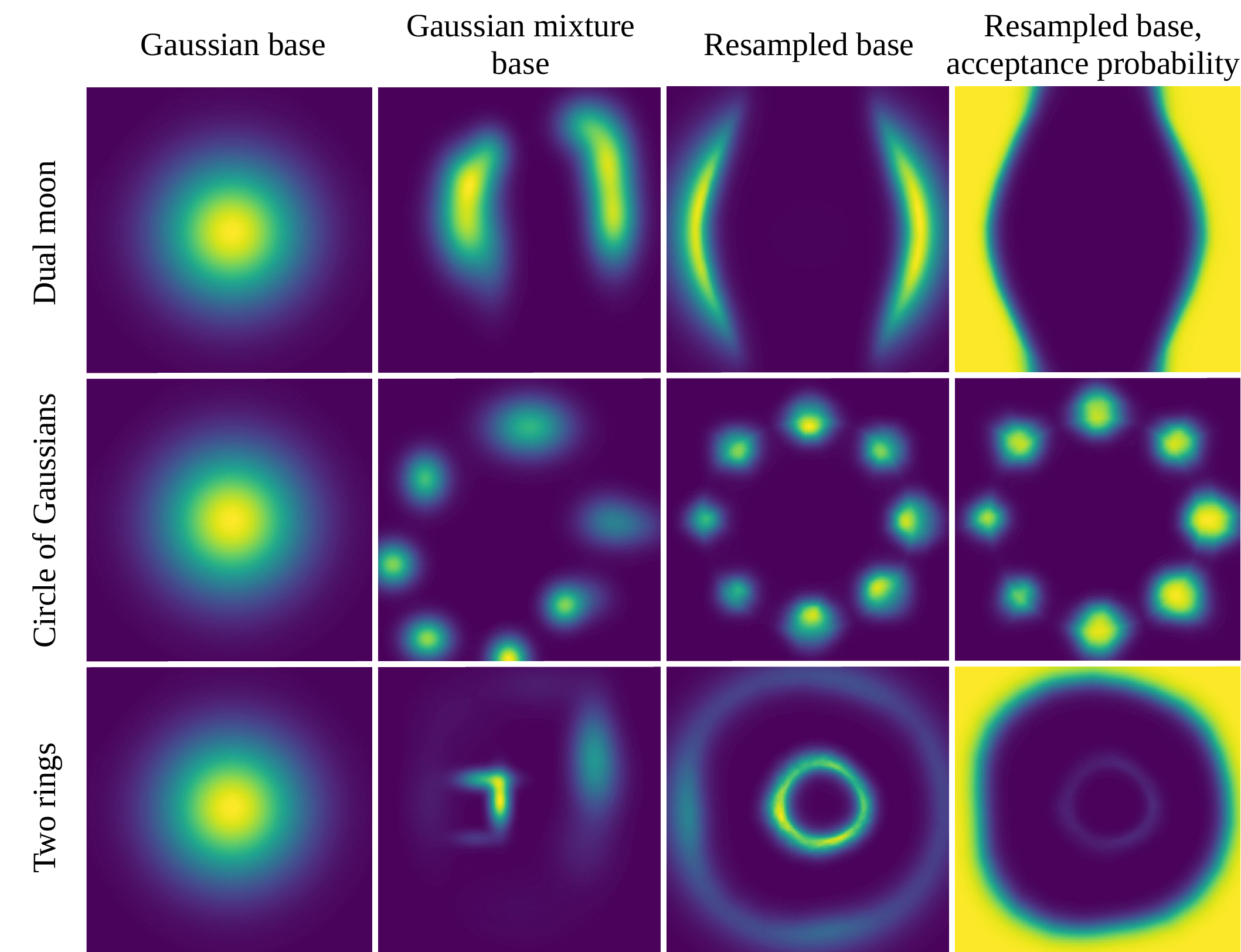}
    \caption{Visualization of the learned base distributions of the residual flow models shown in \autoref{fig:2d_residual_models}.}
    \label{fig:2d_residual_base}
\end{figure}

All models approximating a 2D distribution uses for each layer a fully connected network having 2 hidden layers with 32 hidden units each as parameter map or residual learning block, respectively. The mixture of Gaussian base distributions are initialized by uniformly sampling the mean in the hypercube $[-2.5, 2.5]^D$ and setting the variances to $0.5 \cdot \mathds{1}_D$, where $\mathds{1}_D$ is the $D$-dimensional identity matrix.

The models are trained on a computer with 6 Intel i5-9400F CPUs and a Nvidia GeForce RTX 2070 graphics card. The Adam optimizer with a learning rate of $10^{-3}$ is used. Training is done for $2\cdot10^4$ iterations with a batch size of $1024$.


\section{TABULAR DATA}
\label{sec:tab_app}

In addition to the NSF models, we also trained real NVP models with a Gaussian, a mixture of Gaussians, and a resampled base distribution to the four UCI datasets. The results are shown in \autoref{tab:uci_ll_rnvp}.

\begin{table*}[h]
  \caption{LL of real NVP models with different base distributions on the test sets of the respective datasets. The values are averaged over 3 runs each and the standard error is given as a measure of uncertainty. The highest values within the confidence interval are marked in \textbf{bold}.}
  \label{tab:uci_ll_rnvp}
  \centering
  \vspace{0.3cm}
  \begin{tabular}{l|llll}
    Base distribution & Power & Gas & Hepmass & Miniboone \\
    \hline
    Gaussian & $0.330\pm0.003$ & $10.1\pm0.1$ & $-19.5\pm0.1$ & $-11.65\pm0.05$ \\
    Mixture & $0.341\pm0.001$ & $9.9\pm0.2$ & $-19.5\pm0.1$ & $\mathbf{-11.49\pm0.04}$ \\
    Resampled & $\mathbf{0.560\pm0.006}$ & $\mathbf{12.8\pm0.1}$ & $\mathbf{-18.4\pm0.1}$ & $\mathbf{-11.48\pm0.01}$
  \end{tabular}
\end{table*}

\begin{table*}[h]
  \caption{Details about datasets from the UCI machine learning repository, the architecture of the NSF models as well as the resampled base distribution, and the training procedure.}
  \label{tab:uci_details_nsf}
  \centering
  \vspace{0.3cm}
  \begin{tabular}{l|llll}
     & Power & Gas & Hepmass & Miniboone \\
    \hline
    Dimension & $6$ & $8$ & $21$ & $43$ \\
    Train data points & $1.6\cdot10^6$ & $8.5\cdot10^5$ & $3.2\cdot10^5$ & $3.0\cdot10^4$ \\
    \hline
    Flow layers & $10$ & $10$ & $10$ & $10$ \\
    Hidden layers flow maps & $2$ & $2$ & $2$ & $1$ \\
    Hidden units flow maps & $256$ & $128$ & $256$ & $64$ \\
    \hline
    Hidden layers $a$ & $7$ & $9$ & $4$ & $2$ \\
    Hidden units $a$ & $512$ & $512$ & $512$ & $128$ \\
    Truncation parameter $T$ & $100$ & $50$ & $40$ & $40$ \\
    \hline
    Dropout rate & $0$ & $0.1$ & $0.3$ & $0.3$ \\
    Batch size & $512$ & $512$ & $256$ & $64$ \\
    Learning rate & $3\cdot 10^{-4}$ & $4\cdot 10^{-4}$ & $4\cdot 10^{-4}$ & $3\cdot 10^{-4}$
  \end{tabular}
\end{table*}

\begin{table*}[h]
  \caption{Details about the architecture of the real NVP models used as well as the resampled base distribution, and the training procedure.}
  \label{tab:uci_details_rnvp}
  \centering
  \vspace{0.3cm}
  \begin{tabular}{l|llll}
     & Power & Gas & Hepmass & Miniboone \\
    \hline
    Flow layers & $16$ & $16$ & $16$ & $16$ \\
    Hidden layers flow maps & $2$ & $2$ & $2$ & $2$ \\
    Hidden units flow maps & $128$ & $128$ & $64$ & $32$ \\
    \hline
    Hidden layers $a$ & $3$ & $3$ & $3$ & $3$ \\
    Hidden units $a$ & $512$ & $512$ & $512$ & $256$ \\
    Truncation parameter $T$ & $100$ & $100$ & $100$ & $100$ \\
    \hline
    Dropout rate & $0$ & $0.1$ & $0.2$ & $0.2$ \\
    Batch size & $512$ & $512$ & $256$ & $128$ \\
    Learning rate & $5\cdot 10^{-4}$ & $5\cdot 10^{-4}$ & $3\cdot 10^{-4}$ & $3\cdot 10^{-4}$
  \end{tabular}
\end{table*}

In all experiments regarding the UCI datasets, we use dropout both in the neural networks defining the flow map and the acceptance probability function $a$ of the resampled base distribution during training. Adamax is used as an optimizer \citep{Kingma2015}. The experiments are run on machines with 36 Intel Xeon Platinum 9242 CPUs and 128 GB RAM. Further details on the datasets, the flow architecture, and the training procdure are given in \autoref{tab:uci_details_nsf} and \autoref{tab:uci_details_rnvp}.

\section{IMAGE GENERATION}
\label{sec:app_image_generation}

The parameter maps of the Glow models are convolutional neural networks (CNNs) with 3 layers, the first and the last having a kernel size of $3\times3$ and the middle layer of $1\times1$. The number of channels of the middle layer is 512 and those of the other layers is determined by the respective input and output. This is the same architecture as used in \citep{Kingma2018}.

To ensure that each factor of the base distribution has not more than 100 dimensions, we apply a squeeze operation to the feature map of the first level before passing it to the base distribution. Therefore, each channel has a maximum size of $8\times8 = 64$. A CNN with 4 layers, having 32 channels and a kernel size of $3\times3$ each and a fully connected output layer, is used as acceptance function at each level. The convolutions of this CNN are strided with a stride of 2 until the image size is $4\times4$. The normalization constants are updated with 2048 samples per iteration and before evaluating our models we estimated them with $10^{10}$ samples.

Each model is trained for $10^6$ iterations with the Adam optimizer having a learning rate of $10^{-3}$. The learning rate is warmed up linearly over $10^3$ iterations and the batch size is 512. The models with 8, 16, and 32 layers per level are trained in a distributed fashion on 1, 2, and 4 Nvidia Quadro RTX 5000 graphics cards. We apply Polyak-Ruppert weight averaging \citep{Polyak1990,Ruppert1988} with an update rate of $10^{-3}$, where the exponential moving average of the model weights is computed in order to improve the generalization performance on the test set \citep{Izmailov2018}.

\begin{table}[h]
  \caption{Percentage increase in training time and model size when using a resampled instead of a Gaussian base distribution for the models trained in Section \ref{sec:exp_images}.}
  \label{tab:images_overhead}
  \centering
  \vspace{0.3cm}
  \begin{tabular}{l|ll}
    Layers per level & Training time & Model size\\
    \hline
    8 & $4.7\%$ & $1.5\%$ \\
    16 & $15\%$ & $0.75\%$ \\
    32 & $9.1\%$ & $0.38\%$
  \end{tabular}
\end{table}

\section{BOLTZMANN GENERATORS}

\subsection{Coordinate transformation}
\label{sec:coord_transform}

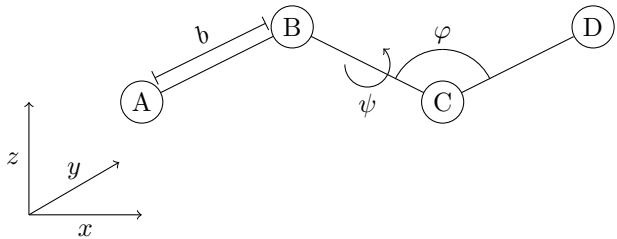
\begin{figure}[h]
	\centering
	\begin{tikzpicture}
		\node (A) at (0,0) {A};
		\node (B) at (2,1) {B};
		\node (C) at (4,0) {C};
		\node (D) at (6,1) {D};
		\draw (A) circle [radius=0.28];
		\draw (B) circle [radius=0.28];
		\draw (C) circle [radius=0.28];
		\draw (D) circle [radius=0.28];
		\draw (A) -- (B) -- (C) -- (D);
		\draw[|-|] (0.17, 0.32) -- node[above,sloped] {$b$} (1.65, 1.05);
		\draw[->] (2.7, 0.5) arc (-180:45:0.3);
		\node at (3, -0.05) {$\psi$};
		\draw (4, 0.7) arc (90:25.9:0.7);
		\draw (4, 0.7) arc (90:154.1:0.7);
		\node at (4, 0.9) {$\varphi$};
		\draw[->] (-1.5, -1.5) -- node[below] {$x$} (0, -1.5);
		\draw[->] (-1.5, -1.5) -- node[left] {$z$} (-1.5, 0);
		\draw[->] (-1.5, -1.5) -- node[above] {$y$} (-0.3, -0.8);
	\end{tikzpicture}
	\caption{Illustration of molecular coordinates. The state of the molecule can be described through the Cartesian coordinates, i.e.\ $x$, $y$, and $z$, of each of the four atoms A, B, C, and D. Alternatively, internal coordinates, i.e.\ bond lengths, bond angles, and dihedral angles, can be used. Here, the bond length $b$ is the distance between atom A and B, the bond angle $\varphi$ being the angle between the bonds between B and C as well as C and D, and the dihedral angle $\psi$ is the angle between the plans spanned by A, B, and C as well as B, C, and D. We use a combination of Cartesian and internal coordinates.}
	\label{fig:mol_coord}
\end{figure}

To simplify the approximation Boltzmann distributions of complex molecules, a coordinate transformation was introduced \citep{Noe2019}. Some of the Cartesian coordinates are mapped to their respective internal coordinates, i.e.\ bond lengths, bond angles, and dihedral angles, which are illustrated in \autoref{fig:mol_coord}. The internal coordinates are normalized, with mean and standard deviation calculated on the training dataset generated through MD, but a suitable experimental dataset could be used as well. To the remaining Cartesian coordinates principal component analysis is applied. Subsequently, the weights of all but the last six principal components are used as coordinates. Thereby, six degrees of freedom are eliminated, corresponding to the three translational and free rotational coordinates which leave the Boltzmann distribution invariant.

\subsection{Setup of the experiments}
\label{sec:aldp_setup}

All real NVP models trained via ML have a neural network with 2 hidden layers and 64 hidden units as a parameter map at each coupling layer. Between the coupling layers, we apply a invertible linear transformation which is learned with the other parameters of the flow, similar to the invertible 1x1 convolutions introduced in \citep{Kingma2018}. The acceptance function of the resampled base distribution is a fully connected neural network with 2 hidden layers having 256 hidden units each. At each iteration, the normalization constant $Z$ is updated with 512 samples from the Gaussian proposal during training. Before evaluating our models, we estimated $Z$ with $10^{10}$ samples. The residual flow models have 8 layers each with each layer having 2 layer fully connected neural network with 64 hidden units and the resampled base distribution has 3 hidden layers with 512 hidden units. All models are trained for $5\cdot10^5$ iterations with the Adam optimizer \citep{Kingma2015} and a batch size of 512. The learning rate is set to $10^{-3}$ and decreased to $10^{-4}$ after $2.5\cdot10^5$ iterations. We also do Polyak-Ruppert weight averaging \citep{Polyak1990,Ruppert1988} with an update rate of $10^{-2}$. Each model is trained and evaluated on a server with 16 Intel Xeon E5-2698 CPUs and a Nvidia GTX980 GPU.

The real NVP models trained by minimizing the KL divergence have the same architecture as those in the previous experiment. However, to improve the stability of the training process, the models are trained with double precision numbers on 32 Intel Xeon E5-2698 CPUs each. $10^5$ iterations are done with the Adam optimizer with a learning rate of $10^{-4}$, which is exponentially decayed every $2.5\cdot10^4$ iterations by a factor of 0.5.

The KL divergences were computed by drawing $10^6$ samples from the model and estimating the respective integrals with histograms.

\subsection{Further results}
\label{sec:aldp_further_results}

As additional performance metric to compare the models, we compute the Ramachandran plots, i.e. a 2D histogram of two dihedral angles. These plots are frequently used to analyse how proteins fold locally and are hence of high importance for many applications. Some Ramachandran plots are show in \autoref{fig:ramachandran_rnvp_fkld}. We also estimate the KL divergences of the ground truth Ramachandran plot obtained from the MD test set and the plots of the models by performing numerical integration with the histograms. The results are given in \autoref{tab:kld_ramachandran_rnvp_fkld}, \autoref{tab:aldp_kld_fkld_resflow}, and \autoref{tab:aldp_kld_rnvp_rkld}.

We also evaluated the stochastic normalizing flow model trained by \cite{Wu2020} through ML on our metrics. The median KL divergences of the marginals is $2.3\cdot10^{-3}$ while the mean is $2.6\cdot10^{-2}$, which is almost an order of magnitude higher that the results of the models with a resampled base distribution. However, the stochastic normalizing flow models the Ramachandran plot very well, where the KL divergence is only $2.4\cdot10^{-1}$. Note that these results have to be taken with a grain of salt, since \cite{Wu2020} used an augmented normalizing flow with less layers than we did. We tried to include their stochastic layers into our models but found training to be very unstable in this setting.

\begin{figure}[h]
    \centering
    \includegraphics[width=\linewidth]{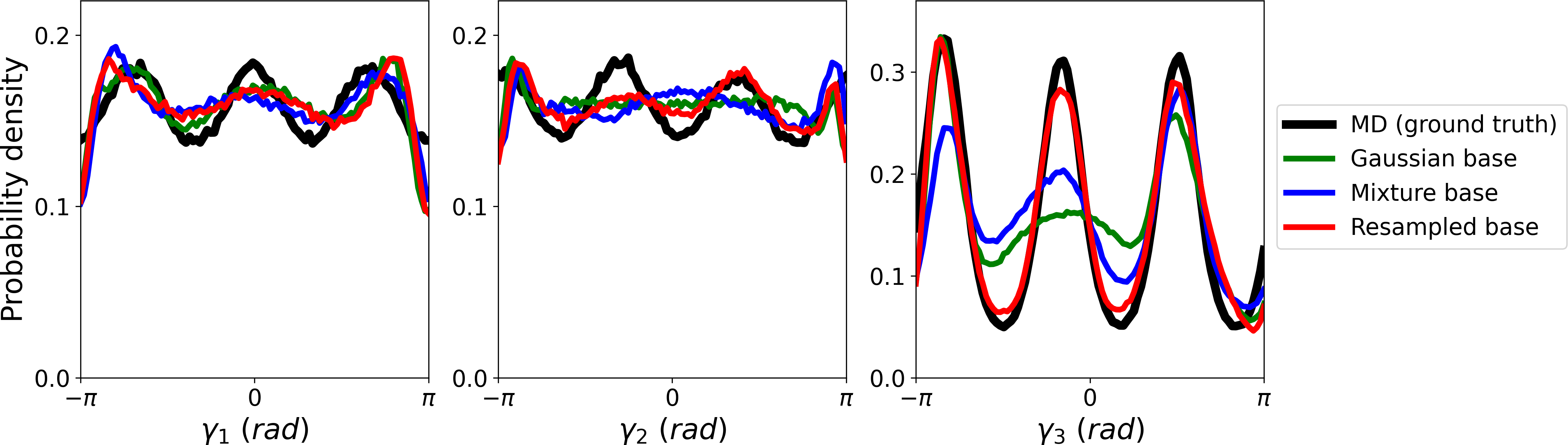}
    \caption{Marginal distribution of three dihedral angles of Alanine dipeptide. The ground truth was determined with a MD simulation. The flow models are based on the residual flow architecture and were trained via ML learning.}
    \label{fig:aldp_resflow_fkld}
\end{figure}

\begin{table}[h]
  \caption{KL divergence of the Ramachandran plots of the MD simulation, serving as a ground truth, and real NVP models trained via ML learning. It was estimated based on a histogram computed from $10^6$ samples.}
  \label{tab:kld_ramachandran_rnvp_fkld}
  \centering
  \vspace{0.3cm}
  \begin{tabular}{l|llll}
    Base distribution & Gaussian & Mixture & Gaussian & Resampled \\
    Number of layers & 16 & 16 & 19 & 16  \\
    \hline
    KL divergence & $4.79\pm0.73$ & $10.8\pm7.3$ & $\mathbf{2.26\pm0.27}$ & $3.00\pm0.36$ 
  \end{tabular}
\end{table}

\begin{table}[h]
  \caption{Quantitative comparison of the residual flow models approximating the Boltzmann distribution of Alanine dipeptide trained via ML learning. The LL is evaluated on a test set obtained with a MD simulation. The KL divergences of the 60 marginals were computed and the mean and median of them are reported. Moreover, the KL divergences of the Ramachandran plots are listed. All results are averages over 10 runs, the standard error is given, and highers LL as well as lowest KL divergences are marked in \textbf{bold}.}
  \label{tab:aldp_kld_fkld_resflow}
  \centering
  \vspace{0.3cm}
  \begin{tabular}{l|lll}
    Base distribution & Gaussian & Mixture & Resampled \\
    \hline
    LL $(\times 10^2)$ & $1.8048\pm0.0002$ & $1.8061\pm0.0002$ & $\mathbf{1.8144\pm0.0002}$ \\
    Mean KLD marginals $(\times 10^{-3})$ & $6.16\pm0.17$ & $31.5\pm1.8$ & $\mathbf{3.49\pm0.15}$ \\
    Median KLD marginals $(\times 10^{-4})$ & $5.21\pm0.12$ & $14.2\pm5.2$ & $\mathbf{4.67\pm0.05}$ \\
    KLD Ramachandran plot & $8.1\pm2.2$ & $25.4\pm10.2$ & $\mathbf{4.4\pm0.9}$
  \end{tabular}
\end{table}

\begin{table}[h]
  \caption{Quantitative comparison of the real NVP models approximating the Boltzmann distribution of Alanine dipeptide trained via the KL divergence. The LL is evaluated on a test set obtained with a MD simulation. The KL divergences of the 60 marginals were computed and the mean and median of them are reported. Moreover, the KL divergences of the Ramachandran plots are listed. All results are averages over 10 runs, the standard error is given, and highers LL as well as lowest KL divergences are marked in \textbf{bold}.}
  \label{tab:aldp_kld_rnvp_rkld}
  \centering
  \vspace{0.3cm}
  \begin{tabular}{l|lll}
    Base distribution & Gaussian & Mixture & Resampled \\
    \hline
    LL $(\times 10^2)$ & $-2.78\pm0.07$ & $-2.70\pm0.04$ & $\mathbf{-1.84\pm0.13}$ \\
    Mean KLD marginals $(\times 10^{-1})$ & $2.91\pm0.05$ & $2.98\pm0.02$ & $\mathbf{2.84\pm0.07}$ \\
    Median KLD marginals $(\times 10^{-3})$ & $4.75\pm0.04$ & $4.77\pm0.03$ & $\mathbf{4.66\pm0.05}$ \\
    KLD Ramachandran plot & $7.63\pm0.18$ & $16.6\pm8.4$ & $\mathbf{6.92\pm0.37}$
  \end{tabular}
\end{table}

\begin{figure}[h]
	\centering
	\subfloat[Real NVP, Gaussian base, 16 layers]{
		\includegraphics[width=0.4\linewidth]{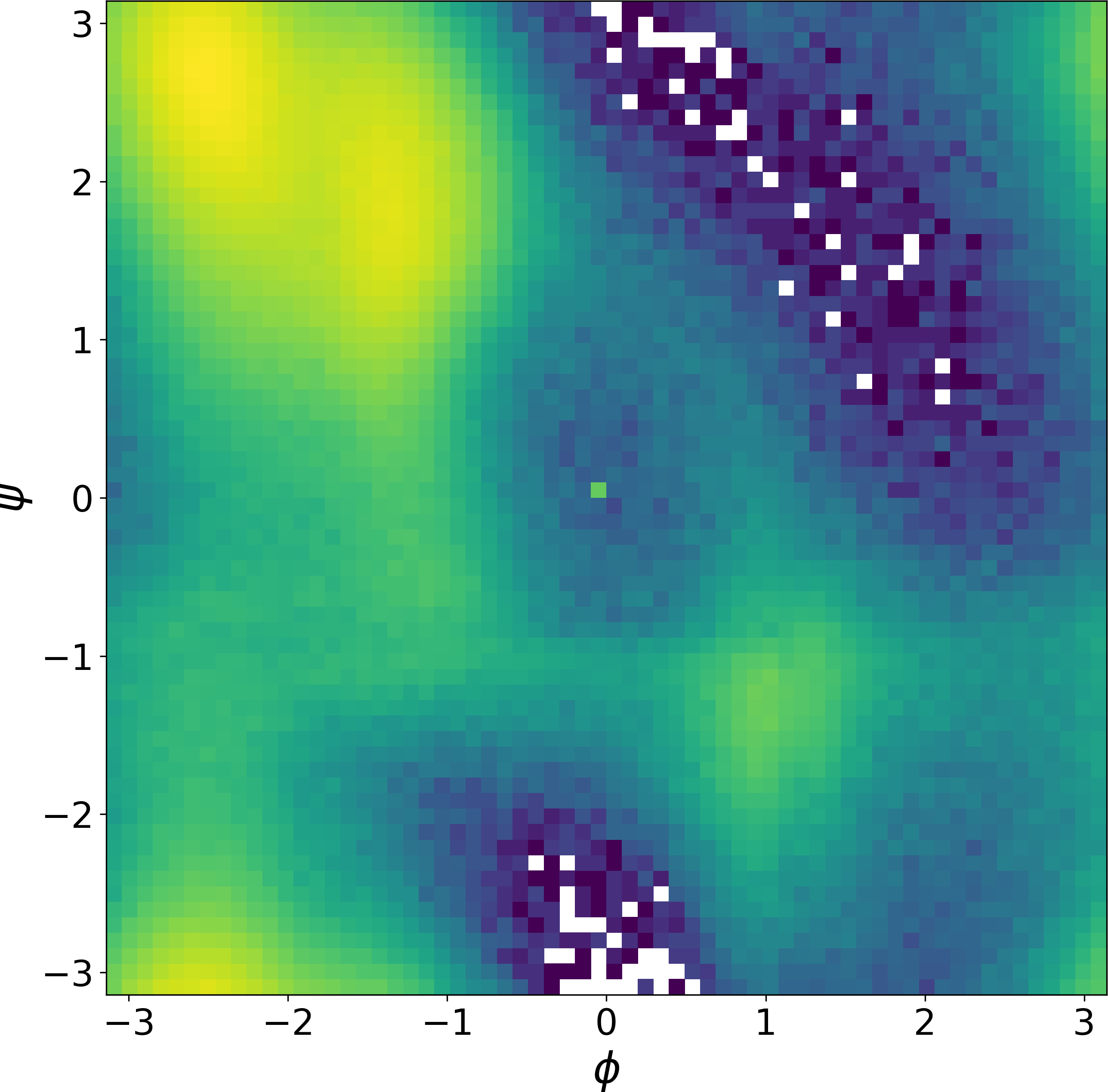}
	}
	\hfil
	\subfloat[Real NVP, Gaussian mixture base, 16 layers]{
		\includegraphics[width=0.4\linewidth]{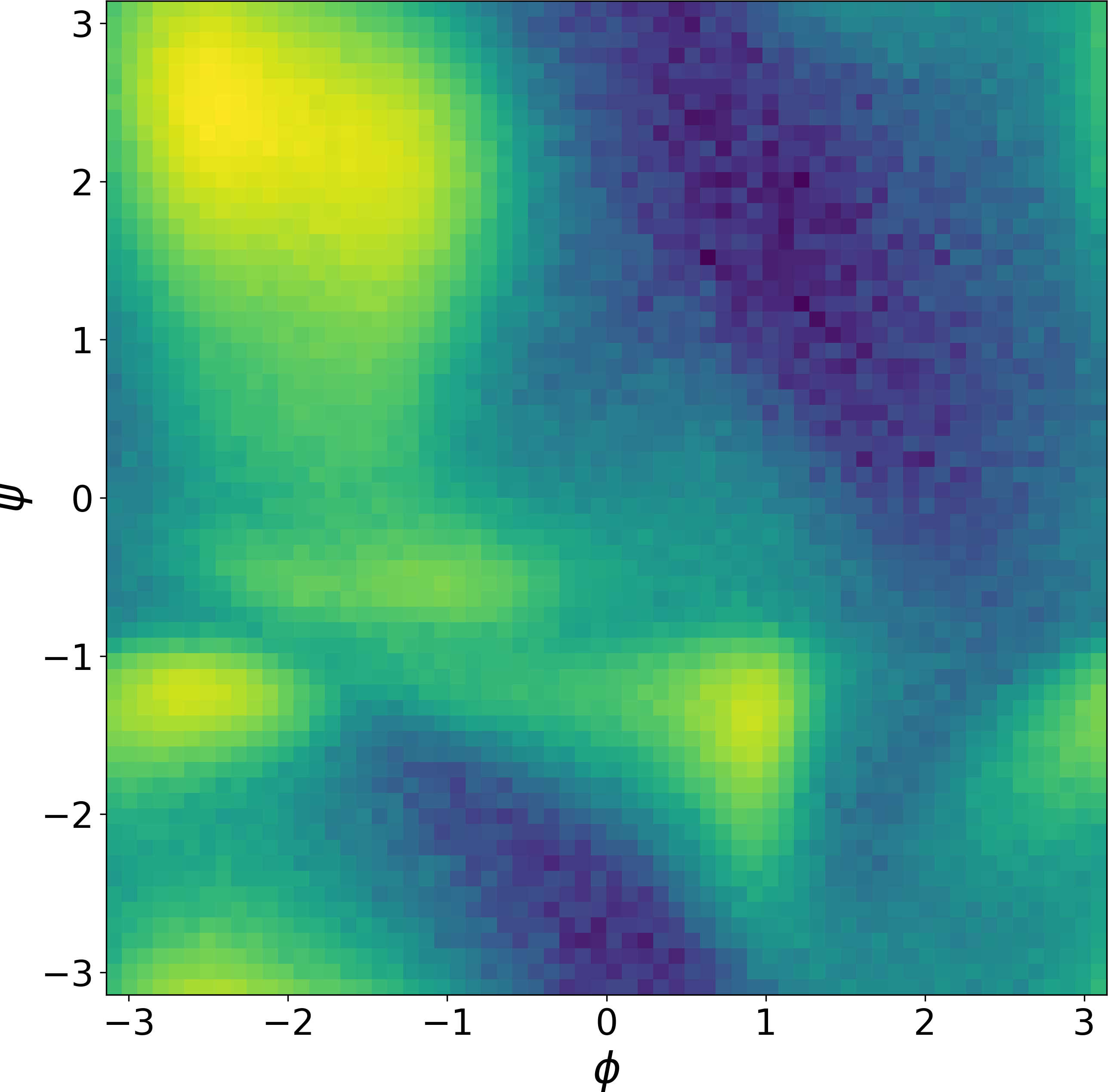}
	}
	\\
	\subfloat[Real NVP, Gaussian base, 19 layers]{
		\includegraphics[width=0.4\linewidth]{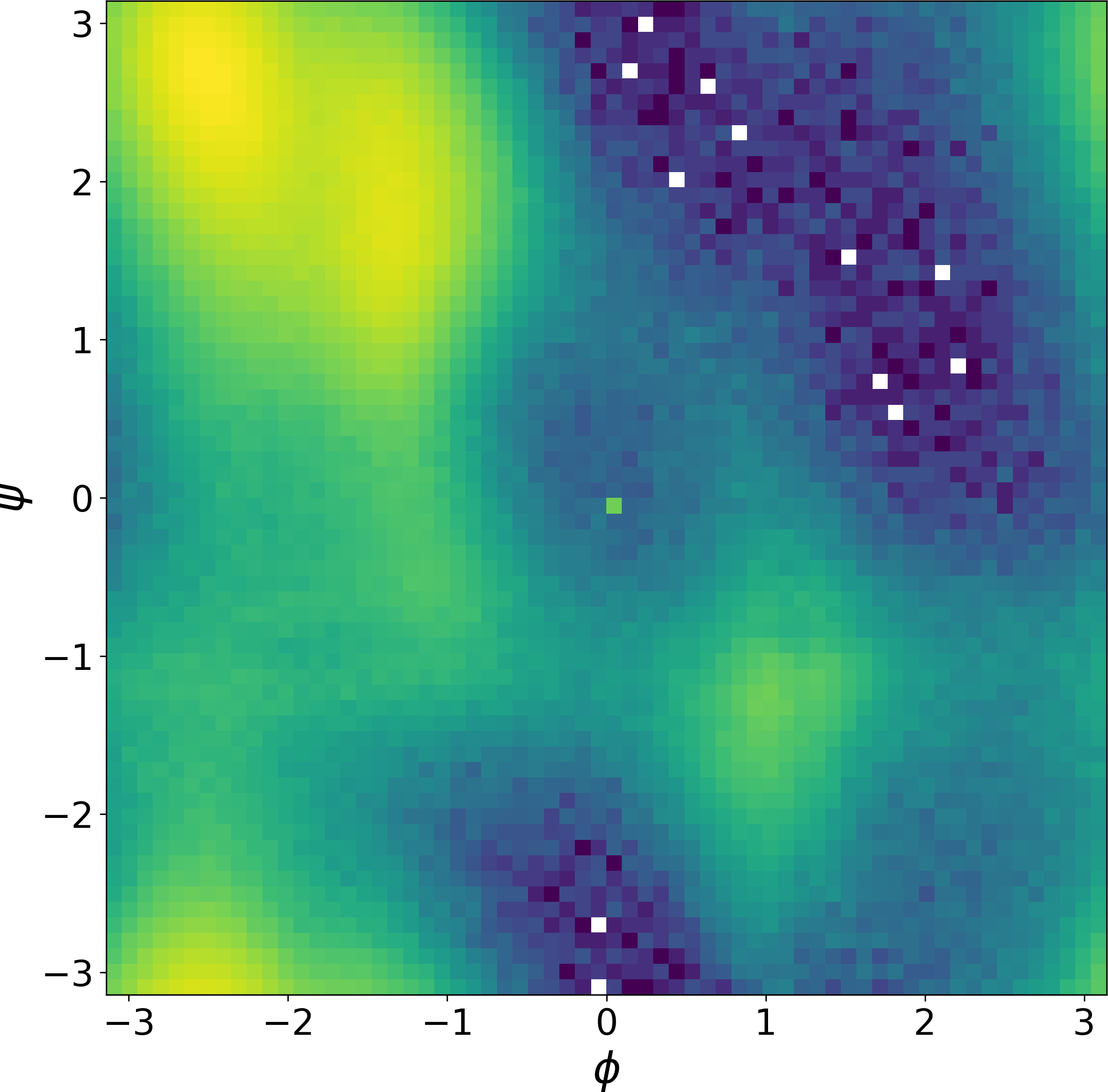}
	}
	\hfil
	\subfloat[Real NVP, Resampled base, 16 layers]{
		\includegraphics[width=0.4\linewidth]{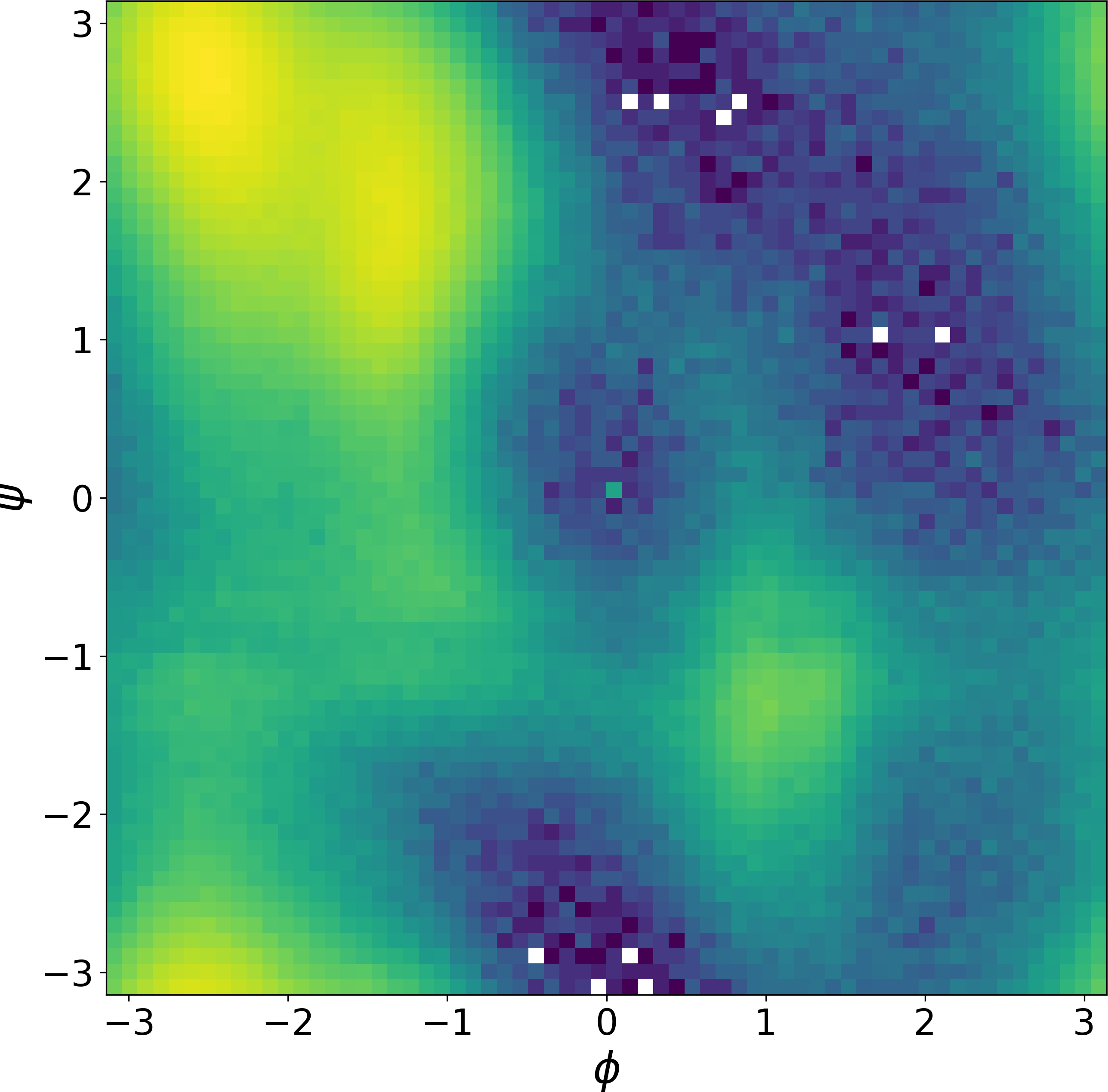}
	}
	\\
	\subfloat[Ground truth (MD simulation)]{
		\includegraphics[width=0.4\linewidth]{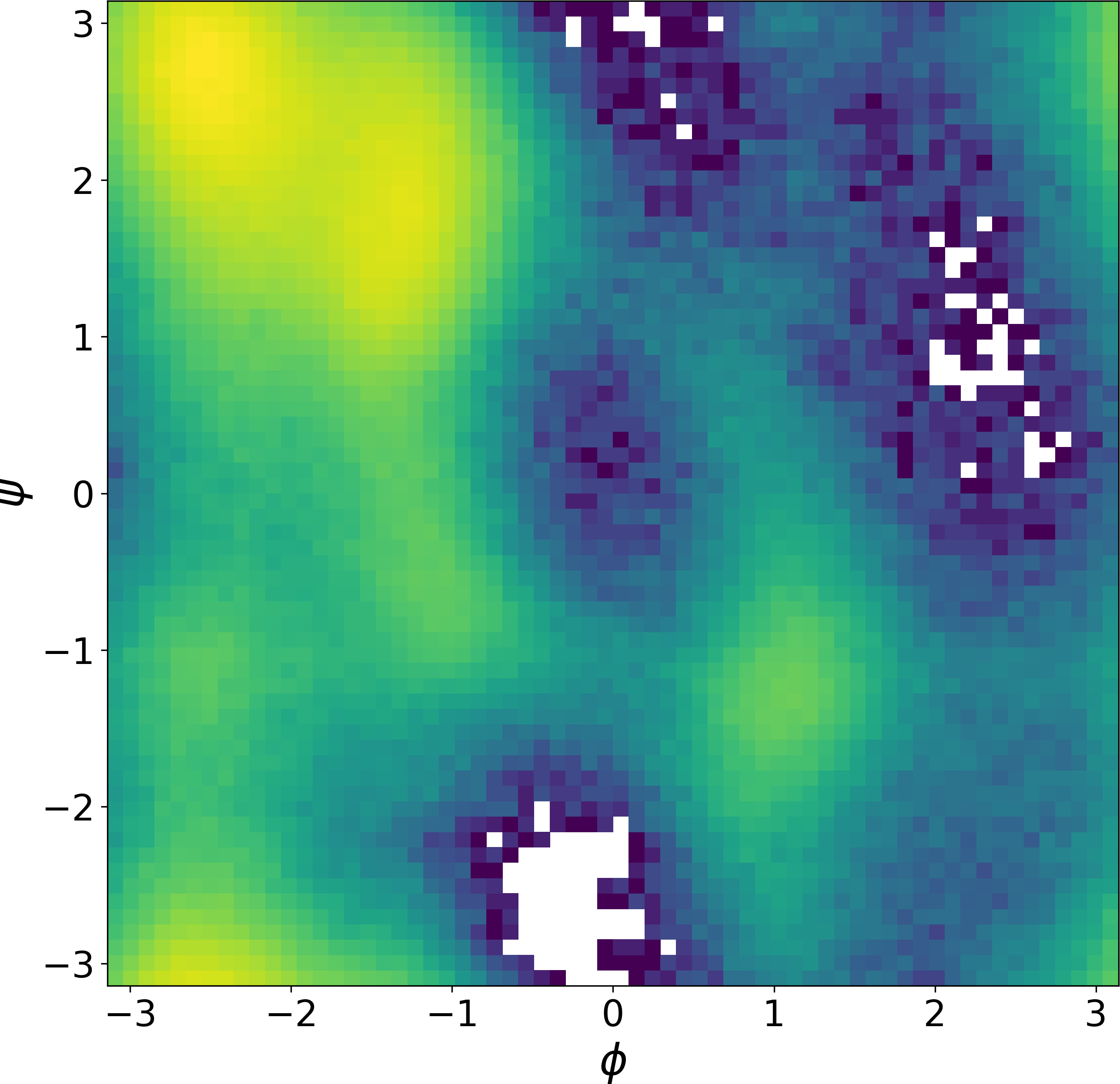}
	}
	\caption{Ramachandran plots of Alanine dipeptide. The flow models were trained via ML learning}
	\label{fig:ramachandran_rnvp_fkld}
\end{figure}

\end{document}